\documentclass[11pt,twoside]{article}

\usepackage[T1]{fontenc}
\usepackage[utf8]{inputenc}
\usepackage{microtype}

\usepackage{amsmath}
\usepackage{amssymb}
\usepackage{amsthm}

\usepackage{graphicx}
\usepackage{booktabs}
\usepackage{array}
\usepackage{float}

\usepackage[margin=1.25in]{geometry}

\usepackage{caption}
\usepackage{enumitem}

\usepackage{authblk}

\usepackage{xcolor}

\usepackage{hyperref}
\hypersetup{hidelinks}

\usepackage{tocloft}
\usepackage{etoc}
\newtheorem{theo}{Theorem}[section]
\newtheorem{lem}[theo]{Lemma}

\newtheorem{cor}[theo]{Corollary}

\theoremstyle{definition}

\theoremstyle{plain}
\newtheorem{rem}[theo]{Remark}

\makeatletter
\newcommand{\@giventhatstar}[2]{#1\;\middle|\;#2}
\newcommand{\@giventhatnostar}[3][]{#1#2\;#1|\;#3#1}
\newcommand{\giventhat}{%
  \@ifstar\@giventhatstar\@giventhatnostar
}
\makeatother

\title{Extreme classification: beating chance with \\ one training example from each class
} 
\author[1]{Kevin Bleakley}
\author[2]{Aaditya Ramdas}

\affil[1]{Inria and Laboratoire de Mathématiques d'Orsay, France}
\affil[2]{Department of Statistics, Stanford University, Stanford, CA, USA}

\begin{document}

\maketitle

\begin{abstract}

\noindent We study a minimal classification problem: Given independent labeled observations $X\sim P$ and $Z\sim Q$ from two unknown distributions $P,Q$, and given an independent target $Y$ drawn with equal probability from $P$ or $Q$, can one classify $Y$ strictly better than chance whenever $P\neq Q$? The one-nearest-neighbor rule succeeds for every pair of multivariate Gaussian distributions with distinct means and a common positive-definite covariance matrix but can perform strictly worse than chance even for smooth densities on the real line. We construct a fixed randomized kernel rule whose expected accuracy is exactly $1/2+\operatorname{MMD}_k^2(P,Q)/4$, and obtain characteristic kernels on countably generated measurable spaces from countable families of measurable binary questions. 
We also prove that a deterministic order rule on $\mathbb R$ beats chance for every pair of distinct Borel probability measures. A measurable encoding then gives a deterministic distribution-free rule which beats chance on every countably generated measurable space, in particular every separable metric space. 
Finally, we show that no rule works for every distinct pair of distributions and every unknown unbalanced class prior; under adaptive target-class selection, every rule other than a fair coin is strictly worse than chance for some finitely supported pair.
\end{abstract}

\section{Introduction}

How little labeled information is needed to classify an observation better than chance? We study perhaps the smallest non-trivial version of this question. Suppose \(P\) and \(Q\) are two distinct, unknown probability distributions on the same measurable space. We observe one labeled reference point drawn independently from each:
$X\sim P$ and
$Z\sim Q$.
We then observe an unlabeled target \(Y\), which is generated from \(P\) with probability \(1/2\) and from \(Q\) with probability \(1/2\). The task is to predict which distribution generated \(Y\), using only the triple \((X,Y,Z)\). Before observing the data, neither distribution is more likely than the other, and an uninformed guess has expected accuracy \(1/2\). Our central question is whether the three observations themselves contain enough information to improve strictly on this baseline whenever \(P\neq Q\).

For this balanced experiment, let $L\in\{P,Q\}$ denote the unobserved target label, with
\[
\mathbb P(L=P)=\mathbb P(L=Q)=\tfrac12,
\qquad L\ \text{independent of }(X,Z).
\]
Conditional on $L$, the target $Y$ is a fresh draw from its corresponding distribution, independent of $X$ and $Z$. Any random seed used by the classifier is independent of this entire experiment. Unless otherwise stated, accuracy is averaged over $X,Z,L,Y$ and the classifier's randomization, not conditional on the realized references. When $P=Q$, the symbols $P$ and $Q$ still denote two distinct class labels with the same observation law.

The balanced prior removes the possibility of beating $1/2$ merely by exploiting class frequencies. If the class probabilities are known and unequal, always predicting the more likely class already beats $1/2$. 
Our principal existence claim is that there exists a classifier, fixed independently of $P$ and $Q$, that beats chance:
\[
\exists\,\text{ classifier }\mathcal C:\quad\forall\,P\neq Q,
\qquad \operatorname{Accuracy}(\mathcal C;P,Q)>\tfrac12,
\]
where all products carry their product sigma-algebras and the measurable rule is fixed for the underlying space independently of $P$ and $Q$. We prove this on every countably generated measurable space, including every separable metric space, for certain randomized and deterministic rules $\mathcal C$. In contrast, no fixed rule has this guarantee simultaneously for all distinct distributions and all unknown unbalanced priors on the class of $Y$. An adversary allowed to select the target class after observing $X$ and $Z$ also precludes a universal strict improvement; see Section~\ref{subsec:unknown-adversarial}.

Let us start by considering an easier setting in which \(P\) and \(Q\) are completely known. For a deterministic measurable classifier $\mathcal{C}$, let
$A = \{y : \mathcal{C}(y) = P\}$.
Conversely, every measurable set $A$ defines such a classifier by predicting $P$ on $A$ and $Q$ on $A^c$. Under the  \(1/2\)--\(1/2\) class prior, the accuracy of the classifier associated with \(A\) is
\[
\begin{aligned}
\operatorname{Acc}(A)
&=
\frac{1}{2}P(A)
+
\frac{1}{2}Q(A^c)
\\
&=
\frac{1}{2}
+
\frac{1}{2}\bigl(P(A)-Q(A)\bigr).
\end{aligned}
\]
Taking the supremum over all measurable sets \(A\), this gives the Bayes accuracy:
\begin{equation}\label{eq:bayes_acc}
\operatorname{Acc}_{\mathrm{Bayes}}
=
\frac{1}{2}
+
\frac{1}{2}
\sup_A \bigl(P(A)-Q(A)\bigr) = \frac{1}{2}
+
\frac{1}{2}\operatorname{TV}(P,Q),
\end{equation}
where the total variation distance between $P,Q$ is denoted
\[
\operatorname{TV}(P,Q)
:=
\sup_A |P(A)-Q(A)|.
\]
Equivalently, the Bayes error is:
\begin{equation}\label{eq:bayes}
R^*
=
\frac{1}{2}
\bigl(1-\operatorname{TV}(P,Q)\bigr).
\end{equation}
Thus, in the equal-prior setting the advantage over random guessing of an optimal classifier that knows both distributions is exactly one half of the TV distance, and no other rule can do better than this.

The setting we consider here is much more restrictive: \(P\) and \(Q\) are unknown, and only one labeled draw from each is available. In Euclidean space,  one immediately obvious classifier is the one-nearest neighbor (1-NN) rule $\mathcal{C}^{\text{1-NN}}$: predict that \(Y\) came from \(P\) when it is closer to \(X\) than to \(Z\), and predict \(Q\) when it is closer to \(Z\) (with equidistant ties resolved by an independent fair coin throughout this article).
Cover and Hart \cite{cover1967nearest} showed, under mild regularity conditions on the distributions $P$ and $Q$, that the limiting error \(R_{\text{1-NN},\infty}\) of the one-nearest-neighbor classifier satisfies
$
R^*
\leq
R_{\text{1-NN},\infty}
\leq
2R^*(1-R^*)$.
 Substituting Eq.~\ref{eq:bayes} into the upper bound and defining $\operatorname{Acc}_{\text{1-NN},\infty} := 1 - R_{\text{1-NN},\infty}$\, gives
\begin{equation}\label{eq:cover}
\operatorname{Acc}_{\text{1-NN},\infty}
\geq
\frac{1}{2}
+
\frac{1}{2}\operatorname{TV}(P,Q)^2.
\end{equation}
Next, since \(\operatorname{TV}(P,Q)=0\) if and only if \(P(A)=Q(A)\) for every measurable set \(A\), or equivalently if and only if \(P=Q\), it follows that, under Cover and Hart's regularity conditions, whenever \(P\neq Q\), asymptotic 1-NN classification is guaranteed to perform strictly better than chance. The right-hand side of Eq.~\ref{eq:cover} involves squared TV, whereas Eq.~\ref{eq:bayes_acc} is an exact identity involving TV itself. The former is an asymptotic lower bound under the usual growing-training-sample protocol, not a finite-sample identity for our experiment.

Despite the fact that the setting considered here is extremely non-asymptotic  ($n=2$ is fixed) with one labeled observation from each distribution, we would like to know: can we still beat chance with the 1-NN rule---or with other rules? And if so: can we bound the expected accuracy from below like in Equations~\ref{eq:bayes_acc} and~\ref{eq:cover}?
One might hope that the answer is yes for the 1-NN rule: perhaps, whenever \(P\neq Q\), a target observation \(Y\) is on average more likely to lie closer to the observation drawn from its own distribution than to the observation drawn from the other. We will show that this intuition is correct when $P$ and $Q$ correspond to certain Gaussian families with unknown parameters, including all pairs of distinct nondegenerate one-dimensional Gaussian distributions. In general, however, it is false: there exist probability distributions \(P\neq Q\), even on \(\mathbb{R}\), for which the 1-NN rule has expected accuracy strictly less than \(1/2\).

We will see, however, that this is a failure of the 1-NN decision rule rather than the information contained in the observations. 
For each fixed characteristic kernel taking values in $[0,1]$, we give a single randomized similarity-based rule that uses only $(X,Y,Z)$ and has balanced expected accuracy strictly greater than $1/2$ for every pair of distinct Borel probability measures on $\mathbb R^d$.

Here is one such rule: suppose \(k\) is a positive semidefinite kernel taking values in \([0,1]\), where \(k(u,v)\) can be interpreted as a measure of similarity between \(u\) and \(v\): larger values indicate greater similarity. Consider the randomized rule that predicts \(P\) with probability
\[
\psi(X,Y,Z)
:=
\frac{1}{2}
+
\frac{1}{2}
\bigl(
k(X,Y)-k(Y,Z)
\bigr),
\]
and \(Q\) otherwise. The difference
$
k(X,Y)-k(Y,Z)
$
lies in \([-1,1]\) and so \(\psi(X,Y,Z)\) always lies in \([0,1]\) and is  therefore  always a valid probability. This rule is more likely to predict \(P\) when \(Y\) is more similar to the reference point \(X\), more likely to predict \(Q\) when \(Y\) is more similar to \(Z\), and predicts each class with probability \(1/2\) when the two similarities are equal.
We will see that the expected accuracy of this rule is
\begin{equation}\label{eq:MMD_acc}
\operatorname{Acc_{Rand}}
=
\frac{1}{2}
+
\frac{1}{4}\operatorname{MMD}_k^2(P,Q),
\end{equation}
where $\operatorname{MMD}_k^2(P,Q)$ is the (squared) maximum mean discrepancy \cite{gretton2012kernel} between the distributions $P$ and $Q$ with respect to the kernel $k$. 
If the kernel \(k\) is \emph{characteristic} (e.g., Gaussian), we know that
$
\operatorname{MMD}_k(P,Q)=0$ if and only if
$
P=Q$ \cite{gretton2012kernel,sriperumbudur2010hilbert}.
Consequently, for every fixed pair \(P\neq Q\), the randomized rule has expected accuracy strictly greater than \(1/2\). We also see that the right-hand side of Eq.~\ref{eq:MMD_acc} again has the same general form seen in Eq.~\ref{eq:bayes_acc} and Eq.~\ref{eq:cover}: 
\begin{equation}\label{eq:discrepancy}
\frac{1}{2} + \text{const}\cdot discrepancy(P,Q) .
\end{equation}

It turns out that similar results can be obtained on non-Euclidean spaces. On  separable metric spaces, we can construct randomized classifiers based on countable dictionaries of measurable binary ``questions''. Informally, such rules randomly choose a yes--no question about the observations and compare the answer for the target \(Y\) with the answers for the two labeled observations $X$ and $Z$. With an appropriate collection of questions, the resulting classifier again achieves expected accuracy strictly greater than \(1/2\) for every pair of distinct probability measures on the space in question. Examples include distributions over functions and distributions over labeled graphs, which highlights that the basic phenomenon under investigation is not fundamentally about geometric proximity.

Randomization is not necessary for unrestricted existence: the order rule in Section~\ref{subsec:order_rule}, combined with the measurable encoding in Corollary~\ref{cor:deterministic-encoding}, gives a deterministic rule on the same class of spaces. The encoding is an existence construction and need not respect the original geometry or admit an efficient implementation; the kernel and random-question constructions remain useful explicit alternatives.

Connections between classification and two-sample testing also underlie classification-accuracy tests studied by Kim et al.~\cite{kim2021classification}. Here we ask a different, fixed-sample question: whether a rule using exactly one reference from each class has population accuracy strictly above chance, with no separation margin assumed. The MMD identity connects this question to the distributional-discrepancy framework of Gretton et al.~\cite{gretton2012kernel}.

The article is structured as follows. Section~\ref{gaussian} establishes the Gaussian 1-NN results, including for arbitrary common positive-definite covariance matrices, and constructs discrete and smooth counterexamples. Section~\ref{random} proves the exact kernel accuracy identity, compares it with 1-NN, and shows that clipping can destroy the universal guarantee. Section~\ref{beyond} develops random questions on countably generated measurable spaces and gives function and graph examples. Section~\ref{extensions} proves the deterministic real-line result, transfers it to general countably generated spaces, gives a multiclass construction, and studies unknown priors and adaptive target-class selection.

\section{The partial success and general failure of 1-NN}\label{gaussian}

In this section we show that in Euclidean space, the intuitive 1-NN rule beats chance for certain general classes of probability distribution, notably Gaussian distributions. We then show via counter-examples that unfortunately, 1-NN does not beat chance in general.  

\subsection{Gaussian distributions with equal variance}\label{equal2}

Suppose that we are told that $P$ and $Q$ are one-dimensional Gaussian distributions with the same variance, but we do not know the values of the variance or the two means. Formally, suppose that 
\[
X \sim \mathcal{N}(\mu_P,\sigma^2), \qquad Z \sim \mathcal{N}(\mu_Q,\sigma^2),
\] 
and that $Y$ is drawn from the 50--50 mixture of the two:
\[
Y \sim \tfrac12 \mathcal{N}(\mu_P,\sigma^2) + \tfrac12 \mathcal{N}(\mu_Q,\sigma^2).
\]
In this Gaussian setting, to simplify notation and proofs, we write $\phi_{m,\sigma^2}$ for the Gaussian probability density function (pdf) with mean $m$ and variance $\sigma^2$, and $\Phi_{m,\sigma^2}$ for the corresponding cumulative distribution function (cdf).

Even without knowing the actual values of $\mu_P$, $\mu_Q$, or $\sigma^2$, the following result shows that the 1-NN rule is---on average---better than a coin flip if $\mu_P \neq \mu_Q$. (If $\mu_P = \mu_Q$, the two Gaussian distributions are identical and a coin flip is indeed optimal.)

\begin{theo}\label{theo1}
Let $P=\mathcal N(\mu_P,\sigma^2)$ and $Q=\mathcal N(\mu_P+\epsilon,\sigma^2)$, where $\mu_P\in\mathbb R$, $\sigma^2>0$, and $\epsilon\neq0$ are unknown. Draw $X\sim P$ and $Z\sim Q$ independently, and generate the target under the balanced experiment. The 1-NN rule
\begin{equation}\label{eqth1}
\mathcal C^{\mathrm{1\text{-}NN}}(x,y,z)=
\begin{cases}
 P,&|x-y|<|z-y|,\\
 Q,&|x-y|>|z-y|,
\end{cases}
\end{equation}
with fair tie-breaking, has expected accuracy strictly greater than $1/2$.
\end{theo}

\noindent
The proof---deferred to Section~\ref{equal} in the Appendix---first conditions on the distribution generating \(Y\) and, by symmetry, reduces to the case \(L=P\). In that case, the nearest-neighbor event is \((Z-X)(Z+X-2Y)>0\). Since \(X\) and \(Z\) have the same variance, these two jointly Gaussian factors have zero covariance and are therefore independent. Moreover, they have the same non-zero mean, so they are more likely to have the same sign than opposite signs, which leads to the result.

Theorem~\ref{theo1} extends to any dimension with an arbitrary common positive-definite covariance matrix; isotropy is not required.

\begin{cor}\label{cor:common-covariance}
Let $P=\mathcal N(\boldsymbol\mu,\Sigma)$ and $Q=\mathcal N(\boldsymbol\mu+\boldsymbol\epsilon,\Sigma)$ on $\mathbb R^d$, where $d\ge1$, $\Sigma$ is positive definite, and $\boldsymbol\epsilon\neq\mathbf0$. The means and covariance matrix are unknown. Under the balanced experiment, Euclidean 1-NN has expected accuracy strictly greater than $1/2$. In fact, both class-conditional accuracies are equal and strictly greater than $1/2$.
\end{cor}

\noindent The proof is in Section~\ref{general}. It converts the distance comparison into a comparison of squared norms of two independent Gaussian vectors with the same covariance and differently scaled means. Orthogonal diagonalization and stochastic ordering establish a strict inequality. This change of coordinates is used only in the proof: the classifier does not need to know $\Sigma$.

\subsection{Gaussian distributions with unequal variances}\label{different2}

We next ask whether the one-dimensional 1-NN decision rule is still valid if we know that the variances are different for the two Gaussians. Indeed, certain steps in the proof of Theorem~\ref{theo1} fail once the two variances are different.
For instance, the symmetry argument whereby 
\[
P^* := \mathbb{P}\big(| X - Y | < | Z - Y | \,\big|\, L=P\big) > 1/2
\]
would have previously been equivalent to 
\[
P^{**} =\mathbb{P}\big(| X - Y | > | Z - Y | \,\big|\, L=Q\big) > 1/2
\]
no longer holds in general. Indeed, it turns out---surprisingly---to be possible that one of these two terms can in fact be less than $1/2$. This occurs, for example, when
 $X \sim \phi_{0,1}$ and $Z \sim \phi_{0.1,0.5}$ as in Fig.~\ref{TwoGaussians};  the 1-NN rule when  
$Y$ is drawn from $\phi_{0,1}$ is correct only around 44.57\% of the time. That is, if you draw once from $\phi_{0,1}$ and once from  $\phi_{0.1,0.5}$,  a second draw from $\phi_{0,1}$ will---more than half the time---be closer to the point from the other distribution. 
\begin{figure}[htbp!]
\begin{center}
\includegraphics[height = 6.5cm]{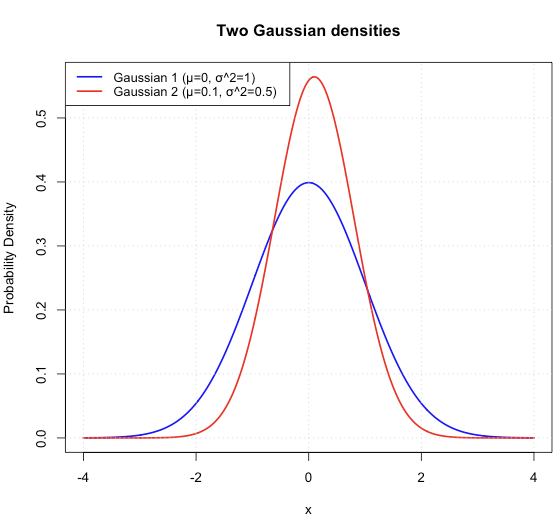}\caption{The densities of $\mathcal N(0,1)$ and $\mathcal N(0.1,0.5)$. The first class's conditional 1-NN accuracy is approximately $0.4457$, whereas the balanced accuracy is approximately $0.5084$; see Section~\ref{unequal}.}
\label{TwoGaussians}
\end{center}
\end{figure}
Nevertheless, their average is strictly greater than $1/2$ whenever the variances differ, even if the means coincide. Thus success in the balanced experiment does not require success within each target class separately.

\begin{theo}\label{theo2}
Let $P=\mathcal N(\mu_P,\sigma_X^2)$ and $Q=\mathcal N(\mu_P+\epsilon,\sigma_Z^2)$, where $\mu_P,\epsilon\in\mathbb R$, $\sigma_X^2>0$, $\sigma_Z^2>0$, and $\sigma_X^2\neq\sigma_Z^2$. Under the balanced experiment, the 1-NN rule \eqref{eqth1} has expected accuracy strictly greater than $1/2$. The parameters are unknown, and $\epsilon=0$ is permitted.
\end{theo}

\noindent Section~\ref{unequal} gives a self-contained proof using bivariate Gaussian sign probabilities. Each class-conditional success probability is minimized when the means coincide, and their average at that minimum exceeds $1/2$ when the variances differ.  Together with Theorem~\ref{theo1}, this proves that 1-NN beats chance for every pair of distinct nondegenerate Gaussian distributions on $\mathbb R$.

\subsection{Counterexample for the 1-NN rule in \texorpdfstring{$
\mathbb{R}^d$}{R\textasciicircum d} for \texorpdfstring{$d \geq 1$}{d >= 1}}\label{counterex}

Unfortunately, the 1-NN rule turns out not to have an expected accuracy greater than $1/2$ for general distributions $P$ and $Q$ in Euclidean space.  
For the following two discrete distributions in $\mathbb R$:
\begin{align*}
    P &= \frac{1}{4}\delta_{2} + \frac{3}{20}\delta_{5} + \frac{3}{5}\delta_{7} , \\
    Q &= \frac{7}{50}\delta_{0} + \frac{6}{25}\delta_{4} + \frac{31}{50}\delta_{7} ,
\end{align*}
exact enumeration with the stated fair tie convention gives
\[
 \mathbb P(\mathrm{Correct}\mid L=P)=\frac{5041}{10000},\qquad
 \mathbb P(\mathrm{Correct}\mid L=Q)=\frac{1498}{3125},
\]
so that
\begin{equation}\label{eq:discrete-counterexample}
 \operatorname{Acc}_{\mathrm{1\text{-}NN}}=\frac{49173}{100000}=0.49173<\frac12.
\end{equation}
For completeness, the class-conditional probabilities of a correct decision given the target value are
\[
\begin{array}{c|ccc}
 y\text{ under }P&2&5&7\\\hline
 \mathbb P(\mathrm{Correct}\mid Y=y,L=P)&\frac{529}{1000}&\frac{91}{200}&\frac{253}{500}
\end{array}
\]
\[
\begin{array}{c|ccc}
 y\text{ under }Q&0&4&7\\\hline
 \mathbb P(\mathrm{Correct}\mid Y=y,L=Q)&\frac{253}{500}&\frac{213}{500}&\frac{247}{500}
\end{array}
\]
Each entry is obtained by summing $P\{x\}Q\{z\}$ over the nine reference pairs, counting a tie as $1/2$. Weighting the entries by the displayed target masses gives \eqref{eq:discrete-counterexample}. The only ties occur when $x=z=7$, whose reference probability is $0.372$.

This failure persists for explicit smooth densities. Let $P_\tau=P*\mathcal N(0,\tau^2)$ and $Q_\tau=Q*\mathcal N(0,\tau^2)$, using the same smoothing variance. Couple each observation to its original atom by adding an independent $\mathcal N(0,\tau^2)$ perturbation. When the reference atoms differ, the possible target atoms are $\{0,2,4,5,7\}$, and every nonzero distance gap
\[
 \bigl||x-y|-|z-y|\bigr|
\]
is an integer at least $1$. There are no midpoint ties: the midpoints of distinct possible reference atoms are $1,3,9/2,5/2,6,7/2,11/2$, none of which is a possible target atom. If all three perturbations have magnitude less than $0.2$, each distance changes by less than $0.4$, so every strict comparison is preserved. When both reference atoms equal $7$, the perturbed references are identically distributed and exchangeable, independently of the target; their expected class-conditional success is exactly $1/2$. A union bound therefore gives
\begin{equation}\label{eq:smoothing-bound}
 \left|\operatorname{Acc}_{\mathrm{1\text{-}NN}}(P_\tau,Q_\tau)-0.49173\right|
 \le 6\Phi_{0,1}(-0.2/\tau).
\end{equation}
Taking $\tau=0.05$ gives accuracy below $0.491931<1/2$: indeed, $\Phi_{0,1}(-4)\le\phi_{0,1}(4)/4$, since $\int_4^\infty\phi_{0,1}(t)\,dt\le\frac14\int_4^\infty t\phi_{0,1}(t)\,dt$. Both mixture densities are smooth and everywhere positive. Their means are $5.45$ and $5.30$, so they remain distinct. Embedding the one-dimensional observations as $(u,0,\ldots,0)$ gives a counterexample in every $d\ge2$; the embedded laws need not have $d$-dimensional densities.

A second counter-example for $d \geq 3$ with an easy geometric failure explanation is provided in Section~\ref{counter} in the Appendix.

\section{Randomized decision rules for arbitrary distributions in \texorpdfstring{$\mathbb{R}^d$}{R\textasciicircum d}}\label{random}

\subsection{Setting and the randomized classifier}

Even though the 1-NN rule can fail in Euclidean space, it turns out that there exist other---randomized---rules which still provably perform better than chance as soon as $P \neq Q$ in $\mathbb{R}^d$ for any $d \geq 1$. The idea is that rather than having an abrupt decision rule based on distances, decisions are smoothed out by probabilistically---rather than deterministically---predicting the class of $Y$ via a Bernoulli trial $\mathcal{B}er(\psi)$, where $\psi$ is a function of $(X,Y,Z)$ that outputs values in $[0,1]$.

Suppose $P$ and $Q$ are two arbitrary distinct probability distributions. We observe three independent random variables as before: $X \sim P$, $Z \sim Q$, and $Y \sim \frac{1}{2}P + \frac{1}{2}Q$.
With the fair tie convention, the 1-NN probability of predicting $P$ is
\[
 \psi_{\mathrm{1NN}}(X,Y,Z)
 =\mathbf1_{\{\|X-Y\|<\|Z-Y\|\}}
  +\tfrac12\mathbf1_{\{\|X-Y\|=\|Z-Y\|\}}.
\]
We replace this probability by a kernel-based randomized rule $\psi(X,Y,Z)$.
 This rule will involve any characteristic kernel that outputs values in $[0,1]$. The main result in Theorem~\ref{theo_randomized} below holds for any such kernel; all that can change is the value of maximum mean discrepancy (MMD) in the final formula (since it goes hand in glove with the chosen kernel), not the actual formula. 
We now define the rule $\psi(X, Y, Z)$ whose output will be the probability with which we predict  $Y$ was generated from $P$:
\begin{equation}
\psi(X, Y, Z) := \frac{1}{2} + \frac{1}{2} \big( k(X, Y) - k(Y, Z) \big).
\label{eq:psi}
\end{equation}
This rule assigns a higher probability to predicting $P$ when $Y$ is more similar---in the eyes of the kernel function---to $X$ than to $Z$. Because $k(u,v) \in [0,1]$, the difference $k(X,Y) - k(Y,Z)$ is in $[-1,1]$. Therefore, $\psi(X,Y,Z) \in [0, 1]$, ensuring it is a valid probability. We now show that this randomized rule performs on average better than chance if $P \neq Q$ in $\mathbb{R}^d$ for any $d \geq 1$.

Here and below a positive semidefinite kernel is real and symmetric. For a bounded measurable kernel, write
\[
 \operatorname{MMD}_k^2(P,Q)
 :=\mathbb E_{P\times P}[k]-2\mathbb E_{P\times Q}[k]
       +\mathbb E_{Q\times Q}[k],
\]
where every pair inside an expectation consists of independent draws. Positive semidefiniteness makes this quantity nonnegative. One direct justification is to average the nonnegative kernel quadratic form of $n$ independent $P$ observations with weights $1/n$ and $n$ independent $Q$ observations with weights $-1/n$, take expectations, and let $n\to\infty$; the diagonal corrections are $O(1/n)$ because $k$ is bounded. A kernel is \emph{characteristic} when this quantity vanishes if and only if $P=Q$; this agrees with injectivity of the kernel mean embedding \cite{sriperumbudur2010hilbert}. Gaussian kernels on $\mathbb R^d$ have this property. The kernel is fixed independently of the observations.

\begin{theo}\label{theo_randomized}
Let $P$ and $Q$ be distinct Borel probability measures on $\mathbb{R}^d$, and let
$k:\mathbb{R}^d\times\mathbb{R}^d\to[0,1]$
be a fixed measurable characteristic positive semidefinite kernel, chosen independently of the observations. Define
\[
\psi(X,Y,Z)
:=
\frac{1}{2}
+
\frac{1}{2}\bigl(k(X,Y)-k(Y,Z)\bigr).
\]
Independently generate
$X\sim P$,\,
$Z\sim Q$, and
$Y\sim 1/2\, P+ 1/2\, Q$.
Consider the randomized classifier that predicts that $Y$ was generated from
$P$ with probability $\psi(X,Y,Z)$ and from $Q$ with probability
$1-\psi(X,Y,Z)$. Then its probability of correct classification is
\[
\mathbb{P}(\mathrm{Correct})
=
\frac{1}{2}
+
\frac{1}{4}\operatorname{MMD}_k^2(P,Q)
>
\frac{1}{2}.
\]
\end{theo}

\noindent \textbf{Proof}.
The overall probability of a correct classification, $\mathbb{P}(\text{Correct})$, is the average of the rule's expected success when $Y$ is drawn from $P$ and when $Y$ is drawn from $Q$:
\begin{equation}
\mathbb{P}(\text{Correct}) = \frac{1}{2} \mathbb{E}[\psi(X, Y, Z) \mid L=P] + \frac{1}{2} \mathbb{E}[1 - \psi(X, Y, Z) \mid L=Q].
\end{equation}
Let $\mathbb{E}_{XXZ}$ denote the expectation over the joint distribution where $X,Y \sim P$ and $Z \sim Q$, and  $\mathbb{E}_{XZZ}$ be the corresponding expectation over the joint distribution where $X \sim P$ and $Y,Z \sim Q$. We evaluate each term separately.

\vspace{0.3cm}
\noindent \textit{Part 1: Expected accuracy when $Y \sim P$.} We have:
\begin{align*}
\mathbb{E}_{XXZ}[\psi(X, Y, Z)] &= \mathbb{E}_{XXZ} \left[ \frac{1}{2} + \frac{1}{2}k(X, Y) - \frac{1}{2}k(Y, Z) \right] \\
&= \frac{1}{2} + \frac{1}{2}\mathbb{E}_{XXZ}[k(X, Y)] - \frac{1}{2}\mathbb{E}_{XXZ}[k(Y, Z)].
\end{align*}
Because $X$ and $Y$ are independent draws from $P$, $\mathbb{E}_{XXZ}[k(X,Y)]$ is the expected kernel value between two independent draws from $P$, denoted $\mathbb{E}_{P \times P}[k]$. Because $Y \sim P$ and $Z \sim Q$, $\mathbb{E}_{XXZ}[k(Y,Z)]$ is the expected kernel value between a draw from $P$ and a draw from $Q$, denoted $\mathbb{E}_{P \times Q}[k]$. Thus, 
\begin{equation}\label{part1}
\mathbb{E}_{XXZ}[\psi(X, Y, Z)] = \frac{1}{2} + \frac{1}{2}\mathbb{E}_{P \times P}[k] - \frac{1}{2}\mathbb{E}_{P \times Q}[k].
\end{equation}

\vspace{0.3cm}
\noindent \textit{Part 2: Expected accuracy when $Y \sim Q$.} 
The probability of correctly guessing $Q$ is $1 - \psi$:
\[
1 - \psi(X, Y, Z) = \frac{1}{2} - \frac{1}{2}k(X, Y) + \frac{1}{2}k(Y, Z).
\]
Taking the expectation under the scenario where $Y \sim Q$:
\[
\mathbb{E}_{XZZ}[1 - \psi(X, Y, Z)] = \frac{1}{2} - \frac{1}{2}\mathbb{E}_{XZZ}[k(X, Y)] + \frac{1}{2}\mathbb{E}_{XZZ}[k(Y, Z)].
\]
Because $X \sim P$ and $Y \sim Q$, $\mathbb{E}_{XZZ}[k(X,Y)] = \mathbb{E}_{P \times Q}[k]$. Because $Y$ and $Z$ are independent draws from $Q$, $\mathbb{E}_{XZZ}[k(Y,Z)] = \mathbb{E}_{Q \times Q}[k]$. Thus, 
\begin{equation}\label{part2}
\mathbb{E}_{XZZ}[1 - \psi(X, Y, Z)] = \frac{1}{2} - \frac{1}{2}\mathbb{E}_{P \times Q}[k] + \frac{1}{2}\mathbb{E}_{Q \times Q}[k].
\end{equation}

\vspace{0.3cm}
\noindent \textit{Part 3: Recombination and MMD.} 
Substituting Equations \ref{part1} and \ref{part2} back into the total probability of correctness, we obtain:
\begin{align*}
\mathbb{P}(\text{Correct}) &= \frac{1}{2} \left[ \frac{1}{2} + \frac{1}{2}\mathbb{E}_{P \times P}[k] - \frac{1}{2}\mathbb{E}_{P \times Q}[k] \right] + \frac{1}{2} \left[ \frac{1}{2} - \frac{1}{2}\mathbb{E}_{P \times Q}[k] + \frac{1}{2}\mathbb{E}_{Q \times Q}[k] \right] \\
&= \frac{1}{4} + \frac{1}{4}\mathbb{E}_{P \times P}[k] - \frac{1}{4}\mathbb{E}_{P \times Q}[k] + \frac{1}{4} - \frac{1}{4}\mathbb{E}_{P \times Q}[k] + \frac{1}{4}\mathbb{E}_{Q \times Q}[k].
\end{align*}
Combining the constants and grouping the kernel expectations then gives:
\begin{equation}\label{MMD_form}
\mathbb{P}(\text{Correct}) = \frac{1}{2} + \frac{1}{4} \Big( \mathbb{E}_{P \times P}[k] - 2\mathbb{E}_{P \times Q}[k] + \mathbb{E}_{Q \times Q}[k] \Big).
\end{equation}
The bracketed term is exactly the squared MMD associated with the kernel $k$:
\[
\text{MMD}^2(P, Q) = \mathbb{E}_{P \times P}[k] - 2\mathbb{E}_{P \times Q}[k] + \mathbb{E}_{Q \times Q}[k].
\]
Substituting this into Equation \ref{MMD_form} yields:
\begin{equation}
\mathbb{P}(\text{Correct}) = \frac{1}{2} + \frac{1}{4} \text{MMD}^2(P, Q).
\end{equation}
Since the kernel $k$ is characteristic, we know that $\text{MMD}^2(P, Q) = 0$ if and only if $P = Q$. Therefore, since 
$P \neq Q$ here, $\text{MMD}^2(P, Q) > 0$. Thus:
$\mathbb{P}(\text{Correct}) > 1/2$. $\square$

\begin{rem}\label{rem:no-uniform-margin}
The guarantee is pointwise in $(P,Q)$, not a uniform margin. Fix $R\neq P$ and set $Q_\eta=(1-\eta)P+\eta R$ for $0<\eta<1$. Then $Q_\eta\neq P$ and bilinearity gives
\[
 \operatorname{MMD}_k^2(P,Q_\eta)=\eta^2\operatorname{MMD}_k^2(P,R).
\]
Thus the kernel rule's advantage tends to zero as $\eta\downarrow0$. More generally, no classifier can have a positive margin uniformly over all distinct pairs, since the Bayes bound \eqref{eq:bayes_acc} gives
\[
 \operatorname{Acc}(\mathcal C;P,Q_\eta)
 \le\tfrac12+\tfrac\eta2\operatorname{TV}(P,R)\longrightarrow\tfrac12.
\]
The independent references cannot improve on the Bayes accuracy for known $P,Q$: conditional on any reference pair, the target still follows the same balanced experiment. 
\end{rem}

\begin{rem}\label{rem:conditional-accuracy}
Theorem~\ref{theo_randomized} averages over the references. For fixed $x,z$, its conditional accuracy is
\[
 \tfrac12+\tfrac14\int\bigl(k(x,y)-k(y,z)\bigr)\,d(P-Q)(y),
\]
which need not exceed $1/2$. For example, on $\{0,1\}$ take $k(u,v)=\mathbf1_{\{u=v\}}$, $P\{1\}=0.8$, and $Q\{1\}=0.2$. For the positive-probability reference pair $x=0,z=1$, the conditional accuracy is $0.2$. Likewise, a bandwidth or kernel chosen using the same observations is not covered automatically by the fixed-kernel identity.
\end{rem}

\subsection{Comparing the randomized rule with 1-NN  in \texorpdfstring{$\mathbb{R}$}{R} for Gaussian distributions}\label{better2}

Though the randomized rule always beats chance, it can still be uniformly worse than the 1-NN rule if we know specific things about $P$ and $Q$. Let us suppose in the following section that $k$ is the standard Gaussian kernel $k(u,v) = \exp\left(-\|u-v\|^2\right)$. Then, if we
consider two one-dimensional Gaussian distributions with different means and the same variance, the 1-NN rule uniformly beats the randomized rule. To show this, we first calculate exactly how much better than chance---the \emph{advantage}---each rule is as a function of $\epsilon$ and $\sigma^2$. These calculations are found in Section~\ref{better} of the Appendix. For the 1-NN rule, we obtain:
\begin{equation}
\text{Advantage}_{\text{1NN}} = \frac{1}{2} \text{erf}\left(\frac{\epsilon}{2\sigma}\right) \text{erf}\left(\frac{\epsilon}{2\sqrt{3}\sigma}\right),
\end{equation}
while for the Gaussian kernel-based randomized rule, we get:
\begin{equation}
\text{Advantage}_{\text{Rand}} = \frac{1 - \exp\left(-\frac{\epsilon^2}{1+4\sigma^2}\right)}{2\sqrt{1+4\sigma^2}}.
\end{equation}
The result is then as follows.

\begin{cor}\label{cor:gaussian-dominance}
For one-dimensional Gaussians with positive mean separation $\epsilon > 0$ and shared variance $\sigma^2 > 0$, the 1-NN rule strictly dominates the randomized classifier in expected probability of success. That is,
\begin{equation}
\text{erf}\left(\frac{\epsilon}{2\sigma}\right) \text{erf}\left(\frac{\epsilon}{2\sqrt{3}\sigma}\right) > \frac{1 - \exp\left(-\frac{\epsilon^2}{1+4\sigma^2}\right)}{\sqrt{1+4\sigma^2}}.
\end{equation}
\end{cor}

\noindent The proof is found in Section~\ref{better} in the Appendix.

\begin{rem}
The comparison holds for every fixed Gaussian-kernel scale $\gamma>0$. Rescale all observations by $\sqrt\gamma$; this leaves their 1-NN comparisons unchanged and turns $\exp(-\gamma(u-v)^2)$ into the standard Gaussian kernel. Equivalently, set $C=(1+4\gamma\sigma^2)^{-1/2}$ in Section~\ref{better}. Since both advantage formulas are even in $\epsilon$, the conclusion also holds for negative nonzero mean separation.
\end{rem}

\begin{rem}
This result does not extend in general to unequal variances. For $P=\mathcal N(0,25)$ and $Q=\mathcal N(1,0.025)$,  $\operatorname{Acc}_{\mathrm{1\text{-}NN}}\approx0.66742825$ while
$ \operatorname{Acc}_{\mathrm{Rand}}
 \approx0.69461934$.
\end{rem}

\subsection{The clipped randomized rule}\label{clipped2}

Recall that the reference version of the randomized rule is based on the probability 
\begin{equation}
\psi(X, Y, Z) := \frac{1}{2} + w \cdot \big( k(X, Y) - k(Y, Z) \big)
\end{equation}
with $w =1/2$. It is easy to see that for $0 < w < 1/2$ the rule still performs better than chance but worse on average than the rule with $w = 1/2$, since 
\begin{equation}
\mathbb{P}(\text{Correct}) = \frac{1}{2} + \frac{w}{2} \cdot \text{MMD}^2(P, Q).
\end{equation}
 The question is: \emph{What happens for} $w > 1/2$?  ``Probabilities'' $\psi$ outside $[0,1]$ become ``possible'' and so one must define a strategy to deal with this. One option is to clip  outputs to $[0,1]$ as follows:
\begin{equation}\label{clipped}
\psi_{clip}(X, Y, Z) := \max \left\{0, \min\{1, \frac{1}{2} + w \cdot \big( k(X, Y) - k(Y, Z) \big) \} \right\}.
\end{equation}
For the standard Gaussian kernel, the rule induced by $\psi_{clip}$ tends to the 1-NN rule as $w \rightarrow \infty$.
\begin{cor}
Assuming the 1-NN rule resolves equidistant ties uniformly at random, and assuming usage of the standard Gaussian kernel in $\mathbb{R}^d$, the expected accuracy of the clipped randomized rule induced by $\psi_{clip}$  tends to the expected accuracy of the 1-NN rule as $w \rightarrow \infty$ for  arbitrary distributions $P$ and $Q$.
\end{cor}
\noindent The proof is deferred to Section~\ref{clippy} in the Appendix.
Consequently, clipping can destroy the distribution-free guarantee. Apply the convergence result to either counterexample in Section~\ref{counterex}. Its limiting accuracy is strictly below $1/2$, so for all sufficiently large finite $w$ the clipped rule also has accuracy below $1/2$.

\section{Beyond Euclidean geometry}\label{beyond}
\label{sec:borel_interrogation}

The core mechanism does not require Euclidean geometry. It applies to every countably generated measurable space $(\Omega,\mathcal B)$; separable metric spaces with their Borel sigma-algebras are important examples.

\subsection{The random-question framework}\label{beats2}
Let $(\Omega,\mathcal B)$ be a measurable space whose sigma-algebra is countably generated. Fix a countable generating $\pi$-system
\[
 \mathcal G=\{A_n\}_{n\ge1}.
\]
Such a system can be formed by taking all finite intersections of a countable family of generators and including $\Omega$. On a separable metric space, finite intersections of elements of a countable topological base give one such choice. Repetitions in an enumeration are harmless.

Let $P$ and $Q$ be distinct probability measures on $(\Omega,\mathcal B)$, and generate $X,Z,L,Y$ under the balanced experiment. Because $\mathcal G$ is a generating $\pi$-system, the uniqueness theorem for finite measures implies that some $n$ satisfies $P(A_n)\neq Q(A_n)$. Interpret $A_n$ as a binary question. Fix a probability mass function $\pi$ on $\mathbb N$ with full support, for example $\pi(n)=2^{-n}$. The canonical random-question classifier proceeds as follows:
\begin{enumerate}
    \item Independently of $X,Y,Z,L$, sample an index $i\sim\pi$, and set
    $A=A_i$.
    \item Evaluate the binary indicators
    $I_Y=\mathbf{1}_{A}(Y)$,
    $I_X=\mathbf{1}_{A}(X)$, and
    $I_Z=\mathbf{1}_{A}(Z)$.
    \item Predict $P$ if $I_Y=I_X\neq I_Z$; predict $Q$ if
    $I_Y=I_Z\neq I_X$; otherwise, predict $P$ or $Q$ according to an
    independent fair coin flip.
\end{enumerate}
This random-question classifier has an expected accuracy above chance.

\begin{theo}
\label{theo:borel_accuracy}

Let $(\Omega,\mathcal B)$ be a countably generated measurable space, let $\mathcal G=\{A_n\}_{n\ge1}$ be a generating $\pi$-system for $\mathcal B$, and let $\pi$ be a probability mass function on $\mathbb N$ with full support. For any distinct probability measures $P\neq Q$ on $(\Omega,\mathcal B)$, under the balanced experiment the expected accuracy of the corresponding random-question classifier is

\[
\mathbb{P}(\mathrm{Correct})
=
\frac{1}{2}
+
\frac{1}{2}
\sum_{n=1}^{\infty}
\pi(n)\bigl(P(A_n)-Q(A_n)\bigr)^2
>
\frac{1}{2}.
\]
\end{theo}
\noindent
The proof is found in Section~\ref{proofbeyond} in the Appendix.
This generalization via countable measurable questions  
links directly to the earlier randomized kernel rule. In fact, the
random-question classifier is exactly the randomized kernel classifier
associated with a particular binary agreement kernel.

\begin{cor}
Under the assumptions of Theorem~\ref{theo:borel_accuracy}, for each
$n\in\mathbb N$ define
\[
k_n(u,v)
:=
\mathbf{1}_{A_n}(u)\mathbf{1}_{A_n}(v)
+
\bigl(1-\mathbf{1}_{A_n}(u)\bigr)
\bigl(1-\mathbf{1}_{A_n}(v)\bigr),
\]
and define the \textit{binary agreement kernel}
\[
k_\pi(u,v)
:=
\sum_{n=1}^{\infty}\pi(n)k_n(u,v).
\]
Conditional on the observations $(X,Y,Z)$, the probability that the
random-question classifier predicts that $Y$ was generated from $P$ is
\[
\frac{1}{2}
+
\frac{1}{2}
\bigl(
k_\pi(X,Y)-k_\pi(Y,Z)
\bigr).
\]
Thus, as a conditional distribution over labels given $(X,Y,Z)$, the random-question classifier is precisely the randomized kernel classifier associated with $k_\pi$. Moreover, $k_\pi$ is $(\mathcal B\otimes\mathcal B)$-measurable, takes values in $[0,1]$, is positive semidefinite and characteristic, and

\begin{equation}
\mathbb{P}(\mathrm{Correct})
=
\frac{1}{2}
+
\frac{1}{4}\mathrm{MMD}_{k_\pi}^2(P,Q).
\label{eq:Pcorrect}
\end{equation}
\end{cor}

\noindent The proof is found in Section~\ref{binary} in the Appendix.

\begin{rem}
The theorem and corollary are also valid when the generating $\pi$-system $\mathcal G$ is finite, with the same proofs except that the infinite sum is replaced by a finite one.
\end{rem}

\begin{rem}
The two-cell partitions $\{A_n,A_n^c\}$ make $k_\pi$ a random-partition kernel in the sense of Davies and Ghahramani~\cite{davies2014random}. Here the full-support, measure-determining dictionary guarantees characteristicness as well as the exact classification identity. A merely generating family need not itself determine measures unless it has suitable closure properties: for example, the uniform laws on $\{(0,0),(1,1)\}$ and on $\{(0,1),(1,0)\}$ have the same probabilities for each coordinate event on $\{0,1\}^2$, although they differ on the intersection $\{(1,1)\}$. This is why finite intersections are included in the examples below.
\end{rem}

\noindent \textbf{Example 1. \emph{Predicting the class of continuous functions}.}
Consider the space
$\Omega=C([0,1])$
of real-valued continuous functions on \([0,1]\), equipped with the supremum metric
\[
d(f,g)=\|f-g\|_\infty
=
\sup_{t\in[0,1]}|f(t)-g(t)|.
\]
This space is separable. Let \(\mathcal B\) denote its Borel
\(\sigma\)-algebra.
Let
$\mathbb Q_{[0,1]}=\mathbb Q\cap[0,1]$.
For each \(q\in\mathbb Q_{[0,1]}\) and \(c\in\mathbb Q\), define
$E_{q,c}
=
\{h\in C([0,1]):h(q)\le c\}$.
Let
\[
\mathcal G
=
\left\{
\bigcap_{i=1}^{k}E_{q_i,c_i}
:\;
k \geq 1,\;
q_i\in\mathbb Q_{[0,1]},\;
c_i\in\mathbb Q
\right\}.
\]
The family \(\mathcal G\) is countable and is a \(\pi\)-system, and
it can be shown that
$\sigma(\mathcal G)=\mathcal B$.
Fix an enumeration
$\mathcal G=\{A_n\}_{n=1}^{\infty}$.
A random question \(A_n\) therefore corresponds to a finite set of 
rational time-threshold pairs \((q_i,c_i)\) and then asking whether a function
\(h\) satisfies all of the inequalities:
\[
h(q_i)\le c_i,
\qquad
i=1,\dots,k.
\]
The answer is ``yes'' if all \(k\) inequalities hold and ``no'' if at
least one fails.

\medskip
\noindent\textbf{Illustrative process model and question sampling.}
To illustrate this framework, we define two probability distributions on
\(C([0,1])\) by linearly interpolating two discrete-time stochastic
processes. Let
$t_j=j/100$ for
$j=0,\dots,100$.
For the first process, let \(\widetilde X\) be a \emph{Gaussian random walk},
initialized at \(\widetilde X_0=0\), with
\[
\widetilde X_j=\widetilde X_{j-1}+\xi_j,
\qquad
\xi_j\overset{\mathrm{i.i.d.}}{\sim}\mathcal N(0,0.25).
\]
For the second process, let \(\widetilde Z\) be a \emph{mean-reverting Gaussian
autoregressive process of order one}, initialized at
\(\widetilde Z_0=0\), with
\[
\widetilde Z_j=0.5\widetilde Z_{j-1}+\eta_j,
\qquad
\eta_j\overset{\mathrm{i.i.d.}}{\sim}\mathcal N(0,0.25).
\]
The two noise sequences are mutually independent. Let
\(\mathcal I(\widetilde X)\) and \(\mathcal I(\widetilde Z)\) denote the
continuous functions obtained by setting
$$\mathcal I(\widetilde X)(t_j)=\widetilde X_j \quad \text{and} \quad
\mathcal I(\widetilde Z)(t_j)=\widetilde Z_j ,$$
and interpolating linearly between consecutive grid points. Define
$P$ as the law of $\mathcal I(\widetilde X)$ and
$Q$ as the law of $\mathcal I(\widetilde Z)$. 
The interpolation map from $\mathbb R^{101}$ into $C([0,1])$ is continuous in the supremum norm, so the two laws are Borel probability measures. They are distinct: at $t_2=0.02$, their variances are $0.5$ and $0.3125$, respectively.

On each trial, we independently generate
$X\sim P$,
$Z\sim Q$, and
$Y\sim 1/2\, P+ 1/2\, Q$,
retaining the component label that generated \(Y\). We also independently
sample a random question \(A\in\mathcal G\) as follows:

\begin{itemize}
    \item \textit{Question complexity.}
    Draw
$k\sim\operatorname{Geom}(0.3)$
    on \(\{1,2,\dots\}\), so that
    $$\mathbb P(k=\ell)=0.3(0.7)^{\ell-1} .$$

    \item \textit{Rational evaluation times.}
For a fully specified dictionary ordering, take $e_1=0$, $e_2=1$, followed by all reduced fractions $a/b\in(0,1)$ in increasing order of $b\ge2$, and then of $a$ within each denominator. Independently for each

    \(i=1,\dots,k\), draw
    $J_i\sim\operatorname{Zeta}(1.5)$
    and set
    $q_i=e_{J_i}$.

    \item \textit{Rational thresholds.}
    Independently for each \(i=1,\dots,k\), draw
    $D_i\sim\operatorname{Zeta}(1.5)$
and
    $V_i\sim\mathcal N(0,2.5^2)$,
    and define
    $N_i$ as  $V_i D_i$ rounded to the nearest integer, and then
    $c_i:= N_i/D_i$.
    
\end{itemize}

Every positive integer $k$ and every rational evaluation time have positive probability. For a rational threshold $c=a/b$, the event $D_i=b$ has positive probability, and rounding $bV_i$ to $a$ also has positive probability because the Gaussian density is positive on the interval $((a-1/2)/b,(a+1/2)/b)$. Thus every finite conjunction in $\mathcal G$ has positive probability, even though different parameter lists can describe the same event. The induced distribution on the dictionary has full support.
For the sampled question
\[
A=\bigcap_{i=1}^{k}E_{q_i,c_i},
\]
the classifier computes
$I_X=\mathbf 1_A(X)$,\,
$I_Y=\mathbf 1_A(Y),$ and
$I_Z=\mathbf 1_A(Z)$.
If
$I_Y=I_X\neq I_Z$,
it predicts \(P\), while if
$I_Y=I_Z\neq I_X$,
it predicts \(Q\). If \(I_X=I_Z\), the selected question does not
distinguish the two reference functions, and the classifier predicts
\(P\) or \(Q\) using an independent fair coin flip.

\medskip

\noindent\textbf{Illustrative trials.}
Figure~\ref{questions} shows nine example triples and sampled questions. These selected panels illustrate how the rule operates; they are not an estimate of its population accuracy for the fully specified enumeration above. The strict population advantage follows from Theorem~\ref{theo:borel_accuracy}. No aggregate simulation rate is inferred from the selected trials.
\begin{figure}[htb!]
\centering
\includegraphics[width=\textwidth]{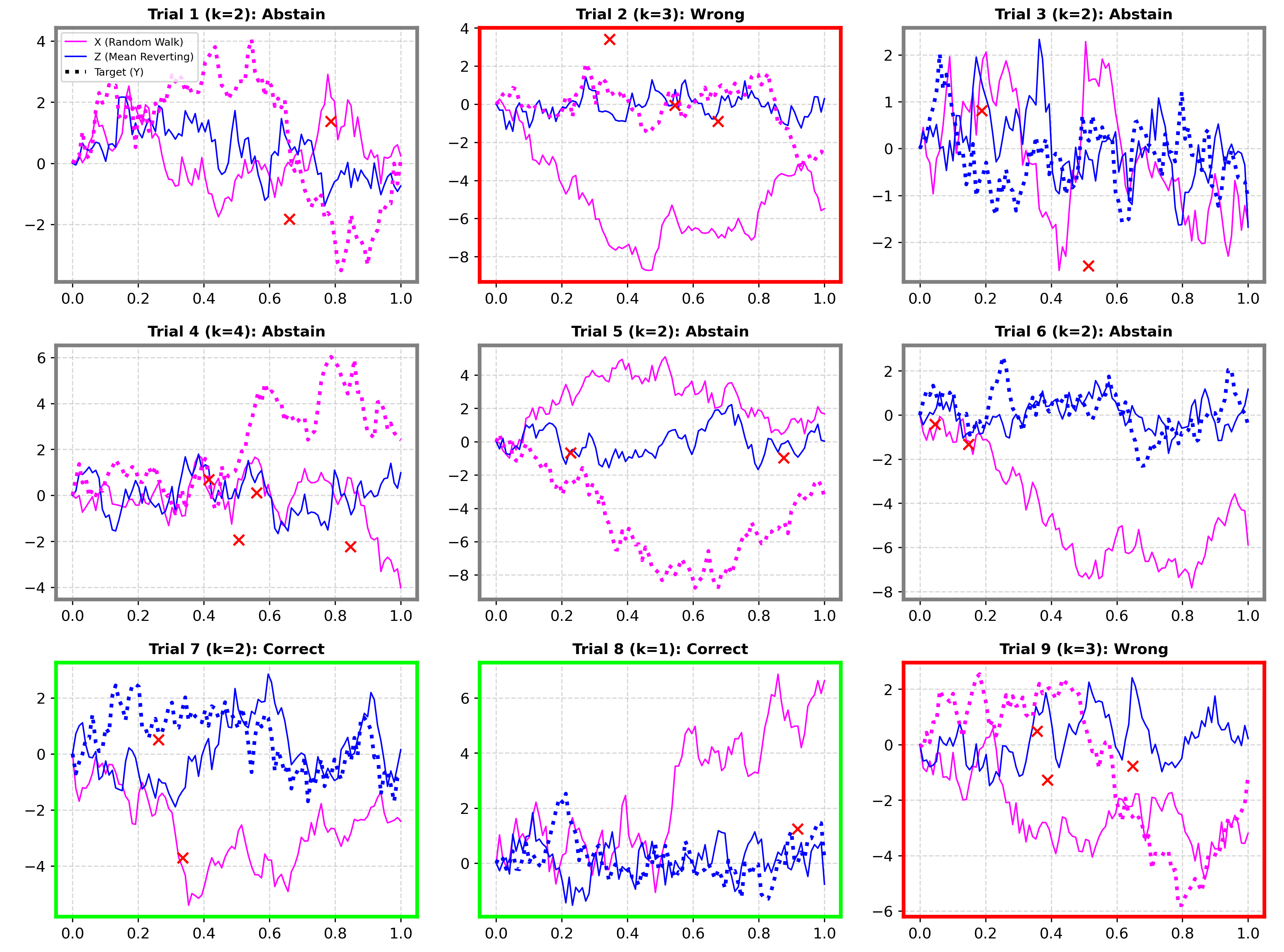}
\caption{Illustrative random-question decisions for continuous functions. Solid curves are the labeled references, the dotted curve is the target, and crosses indicate the sampled time-threshold pairs $(q_i,c_i)$. The decision uses the conjunction of all comparisons $h(q_i)\le c_i$, not an individual cross. ``Abstain'' means that the reference indicators coincide and the final label is chosen by a fair coin. These are qualitative examples, not calibrated pointwise-accuracy plots.}
\label{questions}
\end{figure}
For each curve $h$, the question returns $1$ precisely when all crosses lie on or above the curve at their respective times. In Trial~1, all three indicators are $0$, so the rule uses its fair coin. In Trial~2, the target agrees with the blue reference rather than the pink reference and the displayed prediction is incorrect. In Trials~7 and~8, the reference indicators differ and the target agrees with the correct reference. The outcome of a multi-cross question depends on the entire conjunction, so it cannot be assigned to a single location in the time-threshold plane.\\

\noindent \textbf{Example 2. \emph{Labeled edge-inclusion subgraphs}}. This example, deferred to Section~\ref{Graphs} in the Appendix, shows that the random question framework also applies to edge distribution classes on finite graphs.

\section{Extensions}\label{extensions}

\subsection{Deterministic rules on countably generated spaces}
\label{subsec:order_rule}

Here we show that on countably generated spaces, there also exist \emph{deterministic} universal rules for balanced experiments which beat chance.
We first show that in
\(\mathbb R\), the total order permits a deterministic rule that uses
only the relative order of the three observations. We then highjack this result to prove the existence of universal deterministic rules in any countably generated space. 

In \(\mathbb R\), let us first suppose that
the distributions \(P\) and \(Q\) are absolutely continuous with respect to Lebesgue measure, with
densities \(f\) and \(g\), respectively. 
Write
\[
F(t)=\int_{-\infty}^t f(s)\,ds
\qquad\text{and}\qquad
G(t)=\int_{-\infty}^t g(s)\,ds
\]
for their cumulative distribution functions. Let
\[
X\sim P,
\qquad
Z\sim Q,
\qquad
Y\sim \frac12P+\frac12Q
\]
be generated independently.
Consider the deterministic order rule
\[
C^{\mathrm{ord}}(X,Y,Z)
=
\begin{cases}
P,
&
X<Z\ \text{and}\ Y\leq X,
\\[1mm]
P,
&
X>Z\ \text{and}\ Y>Z,
\\[1mm]
Q,
&
\text{otherwise}.
\end{cases}
\]
Thus, if \(X<Z\), the rule predicts \(P\) if and only if \(Y\leq X\),
whereas if \(X>Z\), it predicts \(P\) if and only if \(Y>Z\).
Since \(P\) and \(Q\) are absolutely continuous, ties occur with
probability zero, and the placement of the weak and strict inequalities
does not affect the rule's accuracy.
The following lemma gives its exact accuracy.
\begin{lem}
\label{lem:order_rule_ac}
Let \(P\) and \(Q\) be absolutely continuous probability distributions
on \(\mathbb R\), with densities \(f\) and \(g\) and cumulative
distribution functions \(F\) and \(G\). Then the deterministic order
rule satisfies
\[
\mathbb P(\mathrm{Correct})
=
\frac12
+
\frac38
\int_{\mathbb R}
\bigl(F(t)-G(t)\bigr)^2
\bigl(f(t)+g(t)\bigr)\,dt.
\]
Equivalently,
\[
\mathbb P(\mathrm{Correct})
=
\frac12
+
\frac38
\int_{\mathbb R}
\bigl(F(t)-G(t)\bigr)^2\,d(P+Q)(t).
\]
In particular, if \(P\neq Q\), then
\[
\mathbb P(\mathrm{Correct})>\frac12.
\]
\end{lem}

\noindent The proof is found in Section~\ref{ordered} of the Appendix. 

\begin{rem}
The assumption of absolute continuity is needed only to write the calculation using densities and ordinary derivatives. The same identity remains valid for arbitrary atomless distributions. Indeed, when $P$ and $Q$ are atomless, their cumulative distribution functions are continuous functions of bounded variation, and the argument carries over by replacing density integrals with Lebesgue–Stieltjes integrals and ordinary integration by parts with integration by parts for continuous functions of bounded variation---see Theorem~\ref{theo:order_rule_general} below.
\end{rem}

\begin{rem}
In the atomless setting above, the advantage over chance can be expressed in
terms of a population Cram\'er--von Mises discrepancy. Let
$H=1/2\,(P+Q)$
denote the equally weighted mixture measure of \(P\) and \(Q\). Then
\[
\int_{\mathbb R}
(F(t)-G(t))^2\,d(P+Q)(t)
=
2\int_{\mathbb R}
(F(t)-G(t))^2\,dH(t),
\]
and hence
\[
\mathbb P(\mathrm{Correct})
= \frac12 +
\frac34
\int_{\mathbb R}
(F(t)-G(t))^2\,dH(t).
\]
The integral on the right is the population analogue of the two-sample Cram\'er--von Mises criterion \cite{anderson1962cramer}, with equal mixture weights. Thus, as before, the classification advantage is
a constant multiple of a distributional discrepancy.
\end{rem}

\noindent When atoms are present, the same calculation acquires additional jump
terms. The exact formula is more cumbersome, so we state here only the
resulting classification guarantee and defer the full statement and proof to
Section~\ref{ordered} of the Appendix.

\begin{theo}
\label{theo:order_rule_general}
Let \(P\) and \(Q\) be arbitrary probability distributions on
\(\mathbb R\). Then the deterministic order rule satisfies
$\mathbb P(\mathrm{Correct})\geq 1/2$
with equality if and only if \(P=Q\). 
\end{theo}

Corollary~\ref{cor:deterministic-encoding} below asserts that we can take advantage of this one-dimensional result to prove that there exist universal deterministic rules which beat chance in balanced experiments in arbitrary countably generated spaces.

\begin{cor}
\label{cor:deterministic-encoding}
Let $(\Omega,\mathcal B)$ be a countably generated measurable space. There is a fixed measurable deterministic classifier on $\Omega^3$ with balanced expected accuracy strictly greater than $1/2$ for every pair of distinct probability measures $P,Q$ on $(\Omega,\mathcal B)$. In particular, this holds for every separable metric space with its Borel sigma-algebra.
\end{cor}
\begin{proof}
Fix a generating sequence $(B_n)_{n\ge1}$ for $\mathcal B$, repeating sets if the generating family is finite, and define the measurable map
\[
 T(x)=\sum_{n=1}^{\infty}\frac{2\mathbf1_{B_n}(x)}{3^n}\in[0,1].
\]
To verify that the encoding retains the whole sigma-algebra, consider
\[
 c:\{0,1\}^{\mathbb N}\longrightarrow[0,1],\qquad
 c(b)=\sum_{n\ge1}\frac{2b_n}{3^n}.
\]
Uniform convergence makes $c$ continuous in the product topology. If two sequences first differ at coordinate $j$, their images differ in absolute value by at least
\[
 \frac2{3^j}-\sum_{n>j}\frac2{3^n}=\frac1{3^j}>0,
\]
so $c$ is injective. The product space is compact; hence, for each $n$, the set
\[
 K_n=c\bigl(\{b:b_n=1\}\bigr)
\]
is compact and therefore Borel in $\mathbb R$. Injectivity gives $T^{-1}(K_n)=B_n$. Thus $\mathcal B\subseteq\sigma(T)$; the reverse inclusion follows from measurability, so $\sigma(T)=\mathcal B$.

Let $P_T=T_\#P$ and $Q_T=T_\#Q$ be the pushforward laws on $\mathbb R$. If $P_T=Q_T$, then $P$ and $Q$ agree on every set $T^{-1}(E)$ with $E$ Borel. These preimages form precisely $\sigma(T)=\mathcal B$, so $P=Q$. Consequently $P\neq Q$ implies $P_T\neq Q_T$.

Apply the fixed order rule of Theorem~\ref{theo:order_rule_general} to $T(X),T(Y),T(Z)$. Its comparison events are measurable, and the transformed variables obey the balanced experiment with laws $P_T,Q_T$. The theorem gives accuracy strictly above $1/2$. Neither the encoding nor the classifier depends on $P,Q$.
\end{proof}

\begin{rem}
The map $T$ need not be injective on $\Omega$; equality of sigma-algebras, rather than a standard-Borel isomorphism assumption, is what the argument uses. The theorem is an existence result, not a claim of continuity, geometric invariance, or computational efficiency. In particular, a total order without suitable measurable structure is not by itself a sufficient hypothesis for the real-line proof.
\end{rem}

\subsection{Multiclass extensions}
\label{subsec:multiclass}

All of the better-than-chance two-class results  in this article extend to the uniform prior multiclass setting.
Suppose that distributions \(P_1,\ldots,P_K\), with \(K\geq 3\), are such that for
each pair \(a<b\), a given binary rule has 
accuracy
\[
\frac{1}{2}+\Gamma_{ab},
\]
where \(\Gamma_{ab}\geq 0\). One labeled observation is drawn
independently from each distribution, and independently of these references the class of the target $Y$ is chosen uniformly from $\{1,\ldots,K\}$, followed by a fresh draw from that class.
Suppose we choose an unordered pair \(\{a,b\}\) uniformly from the
\(\binom{K}{2}\) possible pairs, independently of the observations and target label, and
apply the corresponding binary rule using only the observations from
the two selected classes. The rule can be correct only if the true class
belongs to the selected pair, which occurs with probability \(2/K\).
Conditional on this event, the two selected classes have equal prior
probability. Hence the resulting multiclass accuracy is
\[
\frac{1}{\binom{K}{2}}
\sum_{a<b}
\frac{2}{K}
\left(\frac{1}{2}+\Gamma_{ab}\right)
=
\frac{1}{K}
+
\frac{4}{K^2(K-1)}
\sum_{a<b}\Gamma_{ab}.
\]
Thus, if at least one pairwise advantage is strictly positive, the
multiclass accuracy is strictly greater than \(1/K\), i.e.,  the rule
beats chance.
This construction is randomized because the class pair is sampled, even when the binary rule is deterministic. Its accuracy is at most $2/K$, since the true label must belong to that pair. This may not be optimal, but it immediately supplies a multiclass existence result.

\subsection{Unknown mixture weights and adversarial target-class selection}
\label{subsec:unknown-adversarial}

The preceding results assume that the target class is selected with equal
probability $\alpha = 0.5$ from $P$ or $Q$. We now look at what happens instead when this mixing weight $\alpha$ is fixed but unknown. (If it is fixed, known, and different to $0.5$ we can beat chance by always predicting the most likely class).
Let $\alpha\in(0,1)$ and suppose that, independently of $(X,Z)$,
\[
Y\sim
\begin{cases}
P, & \text{with probability }\alpha,\\
Q, & \text{with probability }1-\alpha.
\end{cases}
\]
We want to know if there exists a rule
that does not depend on $\alpha$, $P$, and $Q$, and yet beats chance whenever $P \neq Q$.
Let
$\psi(x,y,z)\in[0,1]$
denote the probability that the rule predicts that $Y$ was drawn
from $P$. Define
\begin{align}
B_P(\psi;P,Q)
&:=
\mathbb E_{X,Y\sim P,\; Z\sim Q}
\bigl[\psi(X,Y,Z)\bigr],
\\
B_Q(\psi;P,Q)
&:=
\mathbb E_{X\sim P,\; Y,Z\sim Q}
\bigl[1-\psi(X,Y,Z)\bigr].
\end{align}
Thus the classification accuracy under mixing weight $\alpha$ is
\begin{equation}
\operatorname{Acc}_{\alpha}(\psi;P,Q)
=
\alpha B_P(\psi;P,Q)
+
(1-\alpha)B_Q(\psi;P,Q).
\label{eq:alpha-accuracy}
\end{equation}
At $\alpha=1/2$, the preceding results show that a fixed
distribution-free rule can achieve accuracy strictly above $1/2$
whenever $P\neq Q$. This phenomenon does not extend uniformly to an
unknown mixing weight different to 1/2.

\begin{theo}
\label{theo:unknown-alpha}
Let $\Omega$ be a measurable space containing at least two points, with
every singleton measurable. There is no fixed measurable rule
$\psi:\Omega^3\to[0,1]$,
independent of $P$, $Q$, and $\alpha$, such that
$\operatorname{Acc}_{\alpha}(\psi;P,Q)>\frac12$
for every pair $P\neq Q$ and every
$\alpha\in(0,1)\setminus\{1/2\}$. In fact, $P$ and $Q$ may be chosen to be supported on the same
two-point set.
\end{theo}

\noindent The proof, by contradiction, is deferred to Section~\ref{Alpha} in the Appendix.\\

\noindent We next consider what happens when the target class is allowed to depend on the labeled reference
observations. After $X\sim P$ and $Z\sim Q$ are generated independently,
an adversary observes $(X,Z)$ and chooses whether $Y$ will be generated
from $P$ or from $Q$. The target is then drawn with fresh randomness: conditional on the references and selected label, its law is $P$ or $Q$, respectively. The adversary knows $P$, $Q$, and the
classifier's decision rule, and chooses the target distribution to
minimize the classifier's conditional probability of success.
For fixed reference observations $(x,z)$, define
\begin{align}
a_P(x,z)
&:=
\int \psi(x,y,z)\,dP(y),
\label{eq:adv-aP}
\\
a_Q(x,z)
&:=
\int \bigl(1-\psi(x,y,z)\bigr)\,dQ(y).
\label{eq:adv-aQ}
\end{align}
These are the classifier's conditional success probabilities for the two choices. They are measurable functions of $(x,z)$ by bounded measurable integration, so choosing $P$ when $a_P(x,z)\le a_Q(x,z)$ and $Q$ otherwise is a measurable optimal strategy. The classifier's private random seed is not revealed to the adversary. The resulting accuracy is

\begin{equation}
\operatorname{Acc}_{\mathrm{adv}}(\psi;P,Q)
=
\mathbb E_{X\sim P,Z\sim Q}
\left[
\min\{a_P(X,Z),a_Q(X,Z)\}
\right].
\label{eq:adv-master}
\end{equation}
For comparison, under any fixed mixing weight $\alpha$,
\[
\operatorname{Acc}_{\alpha}(\psi;P,Q)
=
\mathbb E_{X\sim P,Z\sim Q}
\left[
\alpha a_P(X,Z)+(1-\alpha)a_Q(X,Z)
\right].
\]
Since
\[
\min\{u,v\}\leq \alpha u+(1-\alpha)v
\]
for every $u,v\in[0,1]$ and $\alpha\in[0,1]$, it follows that
\begin{equation}
\operatorname{Acc}_{\mathrm{adv}}(\psi;P,Q)
\leq
\operatorname{Acc}_{\alpha}(\psi;P,Q)
\qquad
\text{for every }\alpha\in[0,1].
\label{eq:adv-upper-alpha}
\end{equation}
Thus adaptive class selection is at least as difficult for the
classifier as any independent fixed-mixture protocol. In particular,
Theorem~\ref{theo:unknown-alpha} immediately implies that, for every
fixed rule $\psi$, there exist distinct finitely supported $P,Q$ such
that
$\operatorname{Acc}_{\mathrm{adv}}(\psi;P,Q)\leq 1/2$.
Indeed, choose the $\alpha,P,Q$ supplied by
Theorem~\ref{theo:unknown-alpha} and apply
\eqref{eq:adv-upper-alpha}.

We now show that against the above adversary, a fair coin is the only universally safe rule. 

\begin{theo}
\label{theo:adv-no-universal}
Let $(\Omega,\mathcal B)$ contain at least two points and have measurable singletons. Let $\psi:\Omega^3\to[0,1]$ be a fixed measurable rule. If $\psi\not\equiv1/2$, there are distinct finitely supported probability measures $P,Q$ such that
\[
 \operatorname{Acc}_{\mathrm{adv}}(\psi;P,Q)<\tfrac12.
\]
Equivalently, the only fixed rule with adversarial accuracy at least $1/2$ for every distinct finitely supported pair is $\psi\equiv1/2$.
\end{theo}
\noindent The proof is in Section~\ref{adversary}; the converse is immediate because an independent fair coin has success probability $1/2$ for either selected class.

The quantifiers are essential. Theorem~\ref{theo:unknown-alpha} concerns one rule working for all distinct distributions and all unknown unbalanced priors. It does not preclude success on restricted families: Corollary~\ref{cor:common-covariance} gives equal class-conditional accuracies above $1/2$, so its 1-NN rule succeeds for every fixed independent mixture weight $\alpha\in[0,1]$ within that common-covariance Gaussian family.

\section{Conclusion}

One independently drawn labeled reference per class permits a fixed classifier with expected accuracy strictly above $1/2$ for every pair of distinct laws under an independent balanced target. This holds on every countably generated measurable space, both through random questions and through a deterministic encoding of the real-line order rule. The kernel advantage is exactly one quarter of the squared MMD. The order rule identity includes explicit atom corrections and reduces to a population Cram\'er--von Mises formula for atomless laws.

Euclidean 1-NN succeeds for all distinct nondegenerate univariate Gaussian laws and for multivariate Gaussians with distinct means and a common positive-definite covariance, but fails for some smooth densities. Clipping the kernel rule can also destroy its universal guarantee. All positive guarantees average over the references and target, with no uniform margin above chance. They do not extend simultaneously to all unknown class priors or to adaptive target-class selection: in the latter setting, only the fair coin is universally no worse than chance.

\bigskip

\noindent\textbf{Statement on AI use}.
The original 1-NN results underlying this work were obtained prior to any use of AI tools. The principal results developed subsequently---in particular the results proved directly in the main body of the paper---were also obtained and proved without AI assistance. For some of the additional results, whose correctness had first been established by our own calculations, AI tools were used to refine formal proofs. An AI tool suggested the functional data example; indeed, the comparison of Gaussian random walks and mean-reverting Gaussian AR(1) processes is a classical setting in the time-series literature. However, our rule appears original in this context. All mathematical statements, proofs, and examples in the final manuscript were independently verified by the authors, who take full responsibility for their correctness.

\newpage

\appendix
\renewcommand\thesection{\Alph{section}}
\renewcommand\thesubsection{\thesection.\Roman{subsection}}

\section{Appendix}\label{proofs}

\localtableofcontents

\subsection{The Gaussian case with equal variance (from Section \ref{equal2})}\label{equal}

\noindent\textbf{Proof}. The decision rule evaluates correctly if either (1) $Y$ is drawn from $\phi_{\mu_P,\sigma^2}$ and is closer to $X$, or (2) $Y$ is drawn from $\phi_{\mu_P+\epsilon,\sigma^2}$ and is closer to $Z$. By the Law of Total Probability, assuming a uniform prior over the two distributions, the total probability of a correct classification is:
\begin{equation}\label{longone}
\mathbb{P}(\text{Correct}) = \frac{1}{2}\,\mathbb{P}\!\big(\,|X - Y| < |Z - Y|\;\big|\; L=P\big)
+
\frac{1}{2}\,\mathbb{P}\!\big(\,|X - Y| > |Z - Y|\;\big|\; L=Q\big).
\end{equation}

\noindent\textit{Step 1: Symmetry and reduction.}
Let the two conditional probabilities in Equation~\eqref{longone} be denoted as $P^*$ and $P^{**}$ respectively:
\begin{align*}
P^* &:= \mathbb{P}\!\big(\,|X - Y| < |Z - Y|\;\big|\; L=P\big) \\
P^{**} &:= \mathbb{P}\!\big(\,|X - Y| > |Z - Y|\;\big|\; L=Q\big).
\end{align*}

The reflection $R(u)=2\mu_P+\epsilon-u$ exchanges the two Gaussian laws and preserves distances. Conditional on $L=Q$, the transformed triple $(R(Z),R(Y),R(X))$ has the same law as $(X,Y,Z)$ conditional on $L=P$, and the correct-classification event is preserved. Hence $P^*=P^{**}$, and it suffices to prove $P^*>1/2$.

Furthermore, Euclidean distance depends only on the relative distances between points. Uniformly shifting the coordinate system by $-\mu_P$ does not alter the distance inequalities. Thus, without loss of generality, we set $\mu_P = 0$. This defines the three mutually independent variables for the remainder of the proof:
\[
X \sim \mathcal{N}(0, \sigma^2), \quad Y \sim \mathcal{N}(0, \sigma^2), \quad Z \sim \mathcal{N}(\epsilon, \sigma^2).
\]
Our objective simplifies to proving $P^* = \mathbb{P}(|X - Y| < |Z - Y|) > 1/2$.

\medskip
\noindent\textit{Step 2: Algebraic reformulation.}
Since the absolute value metric maps to $\mathbb{R}_{\geq 0}$, and the function $f(x) = x^2$ is strictly monotonically increasing on $\mathbb{R}_{\geq 0}$, we square both sides while preserving the inequality:
\[
|X - Y| < |Z - Y| \iff (X - Y)^2 < (Z - Y)^2.
\]
Rearranging and factoring the difference of squares yields:
\[
(X - Y)^2 - (Z - Y)^2 < 0 \iff \bigl((X - Y) + (Z - Y)\bigr)\bigl((X - Y) - (Z - Y)\bigr) < 0.
\]
Simplifying the terms inside the parentheses gives
$(X + Z - 2Y)(X - Z) < 0$, and
multiplying by $-1$ strictly reverses the inequality:
$(Z - X)(Z + X - 2Y) > 0$.
We define the linear combinations $W := Z - X$ and $U := Z + X - 2Y$. Because $X$, $Y$, and $Z$ are independent continuous random variables, the boundary event where distances are perfectly equal ($W = 0$ or $U = 0$) has a probability of exactly zero. The proposition is thus strictly equivalent to computing $P^* = \mathbb{P}(WU > 0)$.

\medskip
\noindent\textit{Step 3: Joint distribution and independence.}
Because $X, Y, Z$ are mutually independent, normally distributed random variables, the vector $(X, Y, Z)^\top$ forms a multivariate normal distribution. The variables $W$ and $U$ are derived via an affine transformation:
\[
\begin{pmatrix} W \\ U \end{pmatrix} = \begin{pmatrix} -1 & 0 & 1 \\ 1 & -2 & 1 \end{pmatrix} \begin{pmatrix} X \\ Y \\ Z \end{pmatrix}.
\]
Because affine transformations of Gaussian vectors produce Gaussian vectors, $(W, U)^\top$ is a bivariate normal distribution. By the bilinearity of the covariance operator:
\begin{align*}
\mathrm{Cov}(W,U)
 &=\mathrm{Cov}(Z-X,Z+X-2Y)\\
 &=\mathrm{Cov}(Z,Z)+\mathrm{Cov}(Z,X)-2\,\mathrm{Cov}(Z,Y)\\
 &\quad-\mathrm{Cov}(X,Z)-\mathrm{Cov}(X,X)+2\,\mathrm{Cov}(X,Y).
\end{align*}
Because $X, Y, Z$ are mutually independent, they are uncorrelated, meaning $\mathrm{Cov}(A,B) = 0$ for any distinct pair. The equation collapses to the respective variances:
\[
\mathrm{Cov}(W, U) = \mathrm{Var}(Z) - \mathrm{Var}(X) = \sigma^2 - \sigma^2 = 0.
\]
For a bivariate normal distribution, zero covariance is a necessary and sufficient condition for independence. Thus, $W$ and $U$ are independent. 
Using the linearity of expectation and properties of variance for independent variables, the marginal distributions are:
\begin{align*}
\mathbb{E}[W] &= \epsilon - 0 = \epsilon, \quad &\mathrm{Var}(W) &= (1)^2\sigma^2 + (-1)^2\sigma^2 = 2\sigma^2, \\
\mathbb{E}[U] &= \epsilon + 0 - 2(0) = \epsilon, \quad &\mathrm{Var}(U) &= (1)^2\sigma^2 + (1)^2\sigma^2 + (-2)^2\sigma^2 = 6\sigma^2.
\end{align*}
Thus, $W \sim \mathcal{N}(\epsilon, 2\sigma^2)$ and $U \sim \mathcal{N}(\epsilon, 6\sigma^2)$.

\medskip
\noindent\textit{Step 4: Probability evaluation.}
The event $WU > 0$ occurs if and only if $W$ and $U$ share the same strict sign. Due to the independence established in Step 3, the joint probabilities factor into the product of their marginals:
\[
P^* = \mathbb{P}(WU > 0) = \mathbb{P}(W > 0)\mathbb{P}(U > 0) + \mathbb{P}(W < 0)\mathbb{P}(U < 0).
\]
Let $p := \mathbb{P}(W > 0)$. Standardizing the normal variable yields $\mathbb{P}\left(\frac{W - \epsilon}{\sqrt{2}\sigma} > \frac{-\epsilon}{\sqrt{2}\sigma}\right) = \Phi_{0,1}\left(\frac{\epsilon}{\sqrt{2}\sigma}\right)$.
Similarly, let $q := \mathbb{P}(U > 0) = \Phi_{0,1}\left(\frac{\epsilon}{\sqrt{6}\sigma}\right)$.
Substituting these definitions:
\[
P^* = pq + (1-p)(1-q).
\]
To evaluate if $P^* > 1/2$, we subtract $1/2$ and factor:
\begin{align*}
P^* - \frac{1}{2} &= pq + (1 - p - q + pq) - \frac{1}{2} \\
&= 2pq - p - q + \frac{1}{2} \\
&= \frac{1}{2}(2p - 1)(2q - 1).
\end{align*}
The standard normal cumulative distribution function $\Phi_{0,1}(x)$ is strictly monotonically increasing across $\mathbb{R}$, with $\Phi_{0,1}(0) = 1/2$. 
If $\epsilon > 0$, both arguments to $\Phi_{0,1}$ are positive, rendering $p > 1/2$ and $q > 1/2$. Consequently, $(2p - 1) > 0$ and $(2q - 1) > 0$.
If $\epsilon < 0$, both arguments are negative, rendering $p < 1/2$ and $q < 1/2$. Consequently, $(2p - 1) < 0$ and $(2q - 1) < 0$.
By hypothesis, $\epsilon \neq 0$. In all valid cases, $(2p - 1)$ and $(2q - 1)$ possess identical, non-zero signs. Their product is strictly positive.
Therefore, $P^* - 1/2 > 0 \implies P^* > 1/2$. $\square$

\subsection{Common-covariance Gaussians in \texorpdfstring{$\mathbb R^d$}{R\textasciicircum d} (from Section~\ref{equal2})}\label{general}
\begin{proof}[Proof of Corollary~\ref{cor:common-covariance}]
Translate by $-\boldsymbol\mu$ and condition first on $L=P$, so $X,Y\sim\mathcal N(\mathbf0,\Sigma)$ and $Z\sim\mathcal N(\boldsymbol\epsilon,\Sigma)$ independently. Write
\[
 W=Z-X,\qquad U=Z+X-2Y.
\]
The correct-classification event is $W^{\mathsf T}U>0$. The vectors are jointly Gaussian, with
\[
 W\sim\mathcal N(\boldsymbol\epsilon,2\Sigma),\qquad
 U\sim\mathcal N(\boldsymbol\epsilon,6\Sigma),\qquad
 \operatorname{Cov}(W,U)=\Sigma-\Sigma=0.
\]
Thus $W$ and $U$ are independent. Define
\[
 V_+=\frac{W+U/\sqrt3}{\sqrt2},\qquad
 V_-=\frac{W-U/\sqrt3}{\sqrt2},\qquad
 c_\pm=\frac{1\pm1/\sqrt3}{\sqrt2}.
\]
The cross-covariance of $V_+,V_-$ is zero, so these vectors are independent. Each has covariance $2\Sigma$, and its mean is $c_\pm\boldsymbol\epsilon$. Moreover,
\[
 W^{\mathsf T}U=\frac{\sqrt3}{2}
        \bigl(\|V_+\|^2-\|V_-\|^2\bigr).
\]
We verify a scalar comparison explicitly. If $G_m\sim\mathcal N(m,s^2)$ with $s>0$, then for $r>0$ and $m\ge0$,
\[
 \mathbb P(|G_m|\le r)
 =\Phi_{0,1}((r-m)/s)+\Phi_{0,1}((r+m)/s)-1.
\]
For $m>0$ its derivative with respect to $m$ is
\[
 \frac1s\bigl[\phi_{0,1}((r+m)/s)-\phi_{0,1}((r-m)/s)\bigr]<0,
\]
because $|r+m|>|r-m|$. The law of $G_m^2$ depends only on $|m|$; hence increasing $|m|$ strictly decreases its cdf at every positive argument.

Choose an orthogonal matrix $O$ with $O^{\mathsf T}\Sigma O=\operatorname{diag}(\lambda_1,\ldots,\lambda_d)$, where each $\lambda_i>0$, and put $\delta=O^{\mathsf T}\boldsymbol\epsilon$. Orthogonality preserves Euclidean norms. The coordinates of each $O^{\mathsf T}V_\pm$ are independent, have variances $2\lambda_i$, and have means $c_\pm\delta_i$. Since $c_+>c_->0$, every squared plus-coordinate stochastically dominates the corresponding squared minus-coordinate, strictly at each positive argument whenever $\delta_i\neq0$. At least one such coordinate exists.

Replacing the summands one at a time shows that $A=\|V_+\|^2$ stochastically dominates $B=\|V_-\|^2$. This dominance is strict at every positive argument: at a coordinate with $\delta_i\neq0$, condition on the sum $R$ of the other independent nonnegative coordinates. For every $t>0$, the event $R<t$ has positive probability (or $R=0$ if $d=1$), and on that event the strict scalar cdf inequality applies at $t-R>0$. Subsequent replacements preserve the inequality. Thus
\[
 F_A(t)<F_B(t)\qquad(t>0).
\]
Both $A$ and $B$ are atomless, positive almost surely, and independent. Consequently,
\[
 \mathbb P(A>B)=\int(1-F_A(t))\,dF_B(t)
   >\int(1-F_B(t))\,dF_B(t)=\tfrac12.
\]
The last equality follows by comparing two independent, identically distributed, atomless copies of $B$.

Finally, the reflection $v\mapsto\boldsymbol\epsilon-v$ exchanges the two class distributions and preserves distances. Together with exchanging the two reference roles, it shows that the success probability conditional on $L=Q$ equals the one conditional on $L=P$. Both exceed $1/2$, proving the corollary. Ties have probability zero because $A,B$ are independent and atomless.
\end{proof}

\subsection{The Gaussian case with unequal variances (from Section~\ref{different2})}\label{unequal}
\begin{proof}[Proof of Theorem~\ref{theo2}]
Translate so that $\mu_P=0$, and write $a=\sigma_X^2>0$, $b=\sigma_Z^2>0$. Let $p_{a,b}(\epsilon)$ be the 1-NN success probability conditional on $L=P$. Under this conditioning,
\[
 X,Y\sim\mathcal N(0,a),\qquad Z\sim\mathcal N(\epsilon,b)
\]
independently. With $W=Z-X$ and $U=Z+X-2Y$, success means $WU>0$. Set
\[
 v=a+b,\qquad w=5a+b,\qquad c=b-a,\qquad
 \rho=\frac{c}{\sqrt{vw}}.
\]
Then $(W,U)$ is jointly Gaussian with means $(\epsilon,\epsilon)$, variances $(v,w)$, and covariance $c$. Since
\[
 vw-c^2=4a(a+2b)>0,
\]
we have $|\rho|<1$, and the boundary $WU=0$ has probability zero.\\

\noindent\emph{Step 1: Each conditional accuracy is minimized at zero mean separation.}
Let $\Phi_2(s,t;\rho)$ be the cdf of a centered standard bivariate Gaussian with correlation $\rho$, and write $s=\epsilon/\sqrt v$, $t=\epsilon/\sqrt w$. Central symmetry and inclusion--exclusion give
\begin{equation}\label{eq:gaussian-sign-cdf}
 p_{a,b}(\epsilon)
 =1-\Phi_{0,1}(s)-\Phi_{0,1}(t)+2\Phi_2(s,t;\rho).
\end{equation}
The conditional-normal representation
\[
 \Phi_2(s,t;\rho)=\int_{-\infty}^{s}\phi_{0,1}(x)
       \Phi_{0,1}\!\left(\frac{t-\rho x}{\sqrt{1-\rho^2}}\right)\,dx
\]
and its symmetric counterpart give the two partial derivatives of $\Phi_2$. Differentiating \eqref{eq:gaussian-sign-cdf} therefore yields
\begin{equation}\label{eq:gaussian-sign-derivative}
\begin{split}
 p'_{a,b}(\epsilon)
 ={}&\frac{\phi_{0,1}(s)}{\sqrt v}
       \left[2\Phi_{0,1}\!\left(\frac{t-\rho s}{\sqrt{1-\rho^2}}\right)-1\right]\\
 &+\frac{\phi_{0,1}(t)}{\sqrt w}
       \left[2\Phi_{0,1}\!\left(\frac{s-\rho t}{\sqrt{1-\rho^2}}\right)-1\right].
\end{split}
\end{equation}
For $\epsilon>0$, both cdf arguments are strictly positive, because
\[
 \frac{t-\rho s}{\sqrt{1-\rho^2}}
 =\frac{\epsilon(v-c)}{\sqrt v\sqrt{vw-c^2}}>0,
 \qquad
 \frac{s-\rho t}{\sqrt{1-\rho^2}}
 =\frac{\epsilon(w-c)}{\sqrt w\sqrt{vw-c^2}}>0,
\]
with $v-c=2a$ and $w-c=6a$. Hence $p'_{a,b}(\epsilon)>0$ for $\epsilon>0$. Negating all observations proves $p_{a,b}(-\epsilon)=p_{a,b}(\epsilon)$. Thus
\[
 p_{a,b}(\epsilon)\ge p_{a,b}(0),
\]
with strict inequality when $\epsilon\neq0$. Exchanging the classes and translating gives success probability $p_{b,a}(-\epsilon)=p_{b,a}(\epsilon)$ conditional on $L=Q$.

\noindent\emph{Step 2: Average the equal-mean probabilities.}
At $\epsilon=0$, represent the standardized pair as
\[
 (W/\sqrt v,U/\sqrt w)=(G_1,\rho G_1+\sqrt{1-\rho^2}G_2),
\]
where $G_1,G_2$ are independent standard normals. Their joint density depends only on the radius, so their angle is uniform. The two projection directions have angle $\theta=\arccos\rho$. The signs disagree on two sectors of total angle $2\theta$, giving
\[
 p_{a,b}(0)=1-\frac{\arccos\rho}{\pi}
           =\frac12+\frac{\arcsin\rho}{\pi}.
\]
The balanced accuracy at zero separation is therefore
\begin{equation}\label{eq:unequal-zero-mean}
 \frac12+\frac1{2\pi}\left[
 \arcsin\frac{b-a}{\sqrt{(a+b)(5a+b)}}+
 \arcsin\frac{a-b}{\sqrt{(a+b)(a+5b)}}\right].
\end{equation}
If $b>a$, the first arcsine argument is positive and strictly larger than the absolute value of the second, since $5a+b<a+5b$. Oddness and strict monotonicity of $\arcsin$ make the bracket strictly positive. If $a>b$, exchange $a,b$. Thus \eqref{eq:unequal-zero-mean} exceeds $1/2$ whenever $a\neq b$. Step~1 shows that the balanced accuracy for any $\epsilon$ is at least this value. This proves the theorem, including $\epsilon=0$.
\end{proof}

\subsection{1-NN counter-example in \texorpdfstring{$\mathbb{R}^3$}{R\textasciicircum 3} (from Section~\ref{counterex})}\label{counter}

Suppose:
\begin{enumerate}
    \item $P$ is a continuous uniform distribution located entirely on the 2d plane, forming a unit circle centered at the origin:
    \[ P = \mathrm{Uniform}(\{(a, b, 0) \in \mathbb{R}^3 \mid a^2 + b^2 = 1\}). \]
    \item $Q$ is a discrete distribution consisting of two point masses located on the orthogonal third coordinate axis at a distance of $0.6$ from the origin, each with probability $0.5$:
    \[ Q=\tfrac12\delta_{(0,0,0.6)}+\tfrac12\delta_{(0,0,-0.6)}. \]
\end{enumerate}
\begin{figure}[htb!]
\begin{center}
\includegraphics[width=0.8\textwidth]{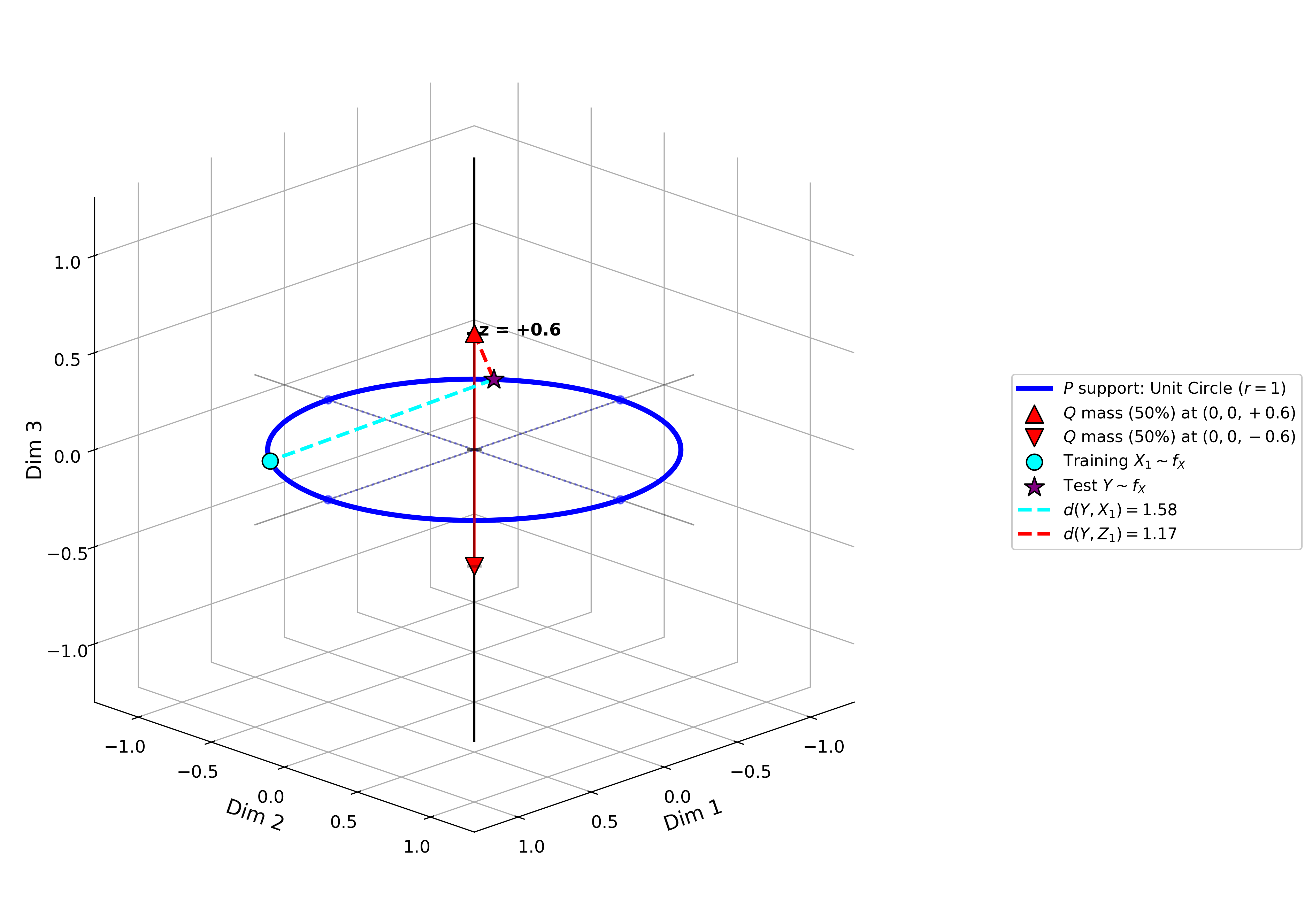}
\caption{Distributions where the 1-NN rule does on average worse than chance  in dimension 3. Points from $P$ are uniformly drawn from the unit circle in the 2d plane, while points from $Q$ are drawn in the third dimension with probability 0.5 from either $(0,0,0.6)$ or $(0,0,-0.6)$. An example trial in which the rule fails is shown: $y$ was drawn from $P$ while $z$ was drawn from the mass at $(0,0,0.6)$.}
\label{counterfigure}
\end{center}
\end{figure}
Figure~\ref{counterfigure} shows what happens.
Essentially, whenever $Y$ is drawn from $P$, less than half the time it will be closer to the point $X$ (on the same circle as $Y$ is) than to the point $Z$ above or below the circle (which is always the same fixed distance to all points on the circle). If instead $Y$ is drawn from $Q$, half the time it will be the same exact point as $Z$, and the other half of the time it will not---and thus closer to \emph{any} point on the circle. The resulting accuracy is exactly $\frac14+\frac{\arccos(0.32)}{2\pi}\approx0.44815854<1/2$, as shown below.\\

\noindent\textbf{Proof}. The expected accuracy of the 1-NN rule is the probability of correct classification, given by the law of total probability:
\[ \mathrm{Acc} = \frac{1}{2} \mathbb{P}(\text{Correct} \mid L=P) + \frac{1}{2} \mathbb{P}(\text{Correct} \mid L=Q) . \]

\noindent \textbf{Case 1: $Y \sim P$}:
The 1-NN rule classifies $Y$ correctly if the squared Euclidean distance to the same-class sample is strictly less than the distance to the opposite-class sample: $d(Y, X_1)^2 < d(Y, Z_1)^2$.
Because $Y$ lies on the unit circle in the $xy$-plane (where the third coordinate is $0$) and $Z_1$ lies on the orthogonal vertical axis at $(0, 0, \pm 0.6)$, the squared distance is deterministic due to the orthogonality of the subspaces:
\[ d(Y, Z_1)^2 = (x - 0)^2 + (y - 0)^2 + (0 - (\pm 0.6))^2 = 1 + 0.36 = 1.36 . \]
The points $Y$ and $X_1$ are independently drawn from the unit circle. The angle $\theta$ between them is uniformly distributed over $[0, \pi]$. By the law of cosines, their squared distance is:
\[ d(Y, X_1)^2 = 1^2 + 1^2 - 2(1)(1)\cos(\theta) = 2 - 2\cos(\theta) . \]
The classification is correct if:
\[ 2 - 2\cos(\theta) < 1.36 \implies 2\cos(\theta) > 0.64 \implies \cos(\theta) > 0.32  .\]
Since $\theta \sim \mathrm{Uniform}([0, \pi])$, the probability of this condition being met is:
\[ \mathbb{P}(\text{Correct} \mid L=P) = \frac{\arccos(0.32)}{\pi} \approx 0.396 . \]

\noindent \textbf{Case 2: $Y \sim Q$}:
Here, the 1-NN rule is correct if $d(Y, Z_1)^2 < d(Y, X_1)^2$. By the same orthogonality argument as in Case 1, the distance to the opposite-class sample is deterministic:
\[ d(Y, X_1)^2 = 1.36 .\]
For the distance to the same-class sample, $Y$ and $Z_1$ are drawn independently from $Q$.
\begin{itemize}
    \item With probability $0.5$, they land on the same pole, yielding $d(Y, Z_1)^2 = 0$.
    \item With probability $0.5$, they land on opposite poles, yielding $d(Y, Z_1)^2 = (2 \times 0.6)^2 = 1.44$.
\end{itemize}
Thus, $d(Y, Z_1)^2 < 1.36$ holds only when $Y$ and $Z_1$ are on the same pole. Therefore:
\[ \mathbb{P}(\text{Correct} \mid L=Q) = 0.5 . \]

\noindent \textbf{Total Expected Accuracy}.
Substituting the conditional probabilities back into the total accuracy equation yields:
\[ \mathrm{Acc} = \frac{1}{2}\left( \frac{\arccos(0.32)}{\pi} \right) + \frac{1}{2}(0.5).\]
Since $\arccos(0.32)<\pi/2$, the exact accuracy is strictly less than $1/2$. The result can be trivially extended to $\mathbb{R}^d$ for $d \geq 4$ by concatenating $d-3$ copies of 0 onto the 3-dimensional data points. 
 \hfill\(\square\)
 
\begin{rem}
For any independent target prior $(\alpha,1-\alpha)$, the exact accuracy in this example is
\[
 \alpha\frac{\arccos(0.32)}\pi+(1-\alpha)\frac12.
\]
This is strictly below $1/2$ for $\alpha>0$ and equal to $1/2$ at $\alpha=0$.
\end{rem}

\subsection{1-NN is always better than the randomized rule (with the standard \mbox{Gaussian} kernel) for equal-variance Gaussians in \texorpdfstring{$\mathbb{R}$}{R} (from Section~\ref{better2})}\label{better}

\subsubsection{Derivation of the 1-NN advantage for equal-variance Gaussians in \texorpdfstring{$\mathbb{R}$}{R}}

As shown in the proof of Theorem~\ref{theo1} in Section~\ref{equal} of the Appendix, the 1-NN rule classifies $Y$ correctly when the product of the linear combinations $W = Z - X$ and $U = Z + X - 2Y$ is strictly positive. Under the equal variance assumption, the covariance between $W$ and $U$ is exactly zero, rendering them independent. Their respective marginal distributions are $W \sim \mathcal{N}(\epsilon, 2\sigma^2)$ and $U \sim \mathcal{N}(\epsilon, 6\sigma^2)$. 
Calculating the probability of correct classification based on these independent marginals yields:
\begin{equation}
P_{\text{1NN}} = \Phi\left(\frac{\epsilon}{\sqrt{2}\sigma}\right)\Phi\left(\frac{\epsilon}{\sqrt{6}\sigma}\right) + \left(1-\Phi\left(\frac{\epsilon}{\sqrt{2}\sigma}\right)\right)\left(1-\Phi\left(\frac{\epsilon}{\sqrt{6}\sigma}\right)\right).
\end{equation}
By subtracting the baseline chance of $0.5$ and converting the standard normal cumulative distribution functions to error functions via the identity $\Phi(z) - 0.5 = \frac{1}{2}\text{erf}(z/\sqrt{2})$, the absolute advantage over chance becomes:
\begin{equation}
\text{Advantage}_{\text{1NN}} = \frac{1}{2} \text{erf}\left(\frac{\epsilon}{2\sigma}\right) \text{erf}\left(\frac{\epsilon}{2\sqrt{3}\sigma}\right).
\end{equation}

\subsubsection{Derivation of the randomized rule's advantage in \texorpdfstring{$\mathbb{R}$}{R} for equal-variance Gaussians}

The randomized rule has probability of correct classification
\[
P_{\mathrm{Rand}}
=
\frac{1}{2}
+
\frac{1}{4}\operatorname{MMD}_k^2(P,Q).
\]
To evaluate the MMD for the Gaussian kernel
$k(u,v)=\exp\bigl(-(u-v)^2\bigr)$,
we must compute expectations of the form
\[
\mathbb{E}[k(A,B)]
=
\mathbb{E}\left[\exp\bigl(-(A-B)^2\bigr)\right],
\]
where $A$ and $B$ are independent Gaussian random variables.
Since the kernel depends on $A$ and $B$ only through their difference, define
$D:=A-B$.
If
$A\sim\mathcal{N}(\mu_A,\sigma_A^2)$
and
$B\sim\mathcal{N}(\mu_B,\sigma_B^2)$
are independent, then
\[
D\sim\mathcal{N}(\mu_D,\sigma_D^2),
\qquad
\mu_D=\mu_A-\mu_B,
\qquad
\sigma_D^2=\sigma_A^2+\sigma_B^2.
\]
Therefore,
\begin{align*}
\mathbb{E}\left[e^{-D^2}\right]
&=
\int_{-\infty}^{\infty}
e^{-d^2}
\frac{1}{\sqrt{2\pi\sigma_D^2}}
\exp\left(
-\frac{(d-\mu_D)^2}{2\sigma_D^2}
\right)\,dd.
\end{align*}
Completing the square in the exponent gives
\[
-d^2-\frac{(d-\mu_D)^2}{2\sigma_D^2}
=
-\frac{1+2\sigma_D^2}{2\sigma_D^2}
\left(
d-\frac{\mu_D}{1+2\sigma_D^2}
\right)^2
-
\frac{\mu_D^2}{1+2\sigma_D^2}.
\]
Hence
\begin{align*}
\mathbb{E}\left[e^{-D^2}\right]
&=
\frac{\exp\left(-\frac{\mu_D^2}{1+2\sigma_D^2}\right)}
{\sqrt{2\pi\sigma_D^2}}
\int_{-\infty}^{\infty}
\exp\left[
-\frac{1+2\sigma_D^2}{2\sigma_D^2}
\left(
d-\frac{\mu_D}{1+2\sigma_D^2}
\right)^2
\right]\,dd \\
&=
\frac{1}{\sqrt{1+2\sigma_D^2}}
\exp\left(
-\frac{\mu_D^2}{1+2\sigma_D^2}
\right).
\end{align*}
In the present equal-variance setting, every difference appearing in the
three MMD terms has variance
\[
\sigma_D^2=2\sigma^2.
\]
Consequently,
\[
\mathbb{E}[k(A,B)]
=
\frac{1}{\sqrt{1+4\sigma^2}}
\exp\left(
-\frac{\mu_D^2}{1+4\sigma^2}
\right),
\]
where the value of $\mu_D$ depends on whether $A$ and $B$ are drawn from the
same distribution or from different distributions.
We now evaluate the three required MMD components:

\begin{enumerate}
    \item \textbf{Internal similarity of $X$}: Let $A, B \sim \mathcal{N}(0, \sigma^2)$. Then $\mu_D = 0$ and
    \[ \mathbb{E}_{P \times P}[k] = \frac{1}{\sqrt{1 + 4\sigma^2}} \exp(0) = \frac{1}{\sqrt{1 + 4\sigma^2}}. \]
    
    \item \textbf{Internal similarity of $Z$}: Let $A, B \sim \mathcal{N}(\epsilon, \sigma^2)$. Then $\mu_D = 0$ and
    \[ \mathbb{E}_{Q \times Q}[k] = \frac{1}{\sqrt{1 + 4\sigma^2}} \exp(0) = \frac{1}{\sqrt{1 + 4\sigma^2}}. \]
    
    \item \textbf{Cross-similarity of $X$ and $Z$}: Let $A \sim \mathcal{N}(0, \sigma^2)$ and $B \sim \mathcal{N}(\epsilon, \sigma^2)$. Then $\mu_D = -\epsilon$ and
    \[ \mathbb{E}_{P \times Q}[k] = \frac{1}{\sqrt{1 + 4\sigma^2}} \exp\left( - \frac{\epsilon^2}{1 + 4\sigma^2} \right). \]
\end{enumerate}
Combining these into the  term $\text{MMD}^2 = \mathbb{E}_{P \times P}[k] - 2\mathbb{E}_{P \times Q}[k] + \mathbb{E}_{Q \times Q}[k]$ gives:
\[
\text{MMD}^2 = \frac{1}{\sqrt{1 + 4\sigma^2}} - 2 \left[ \frac{1}{\sqrt{1 + 4\sigma^2}} \exp\left( - \frac{\epsilon^2}{1 + 4\sigma^2} \right) \right] + \frac{1}{\sqrt{1 + 4\sigma^2}}.
\]
Grouping terms and multiplying by the scaling factor $\frac{1}{4}$ to extract the randomized rule's advantage over chance ($P_{\text{Rand}} - 1/2$) gives:
\begin{equation}
\text{Advantage}_{\text{Rand}} = \frac{1 - \exp\left(-\frac{\epsilon^2}{1+4\sigma^2}\right)}{2\sqrt{1+4\sigma^2}}.
\end{equation}

\subsubsection{Proof of the comparison}
\begin{proof}[Proof of Corollary~\ref{cor:gaussian-dominance}]
Set $x=\epsilon/(2\sigma)>0$, $u=x^2$, and $C=(1+4\sigma^2)^{-1/2}\in(0,1)$. Then $s=1-C^2\in(0,1)$, and the desired inequality is
\[
 \operatorname{erf}(x)\operatorname{erf}(x/\sqrt3)
      >C(1-e^{-su}).
\]
We use the bound $\operatorname{erf}(z)>\sqrt{1-e^{-z^2}}$ for $z>0$, associated with normal-integral inequalities such as those of Chu~\cite{chu1955bounds}. A direct proof follows by comparing a square to its inscribed disk:
\[
\begin{split}
 \operatorname{erf}(z)^2
 &=\frac1\pi\int_{[-z,z]^2}e^{-(v^2+w^2)}\,dv\,dw\\
 &>\frac1\pi\int_{v^2+w^2\le z^2}e^{-(v^2+w^2)}\,dv\,dw
 =1-e^{-z^2}.
\end{split}
\]
Strictness holds because the square outside the disk has positive area and the integrand is positive.
Write $f(t)=1-e^{-t}$ for $t\ge0$. The preceding bound gives
\[
 \operatorname{erf}(x)\operatorname{erf}(x/\sqrt3)
 >\sqrt{f(u)f(u/3)}.
\]
We have $f(u)>f(su)$ by strict monotonicity. We also claim that $f(u/3)>C^2f(su)$. If $s\le1/3$, this follows from $f(u/3)\ge f(su)>0$ and $C^2<1$. If $s>1/3$, concavity of $f$ and $f(0)=0$ imply
\[
 f(u/3)=f\!\left(\frac{su}{3s}\right)
 \ge\frac1{3s}f(su)>C^2f(su),
\]
since $3sC^2=3C^2(1-C^2)\le3/4<1$. Multiplying these positive inequalities yields
\[
 f(u)f(u/3)>C^2f(su)^2,
\]
which proves the required comparison after taking square roots.
\end{proof}

\subsection{Convergence of the clipped rule to the 1-NN rule (from Section~\ref{clipped2})}\label{clippy}

\noindent\textbf{Proof}.
For each $w>0$, define
\[
\psi_{\mathrm{clip},w}(X,Y,Z)
:=
\max\left\{
0,
\min\left\{
1,
\frac{1}{2}
+
w\bigl(k(X,Y)-k(Y,Z)\bigr)
\right\}
\right\},
\]
and let
\[
\Delta:=k(X,Y)-k(Y,Z).
\]
For every fixed realization of $(X,Y,Z)$, as $w\to\infty$ we have
\[
w\Delta\to
\begin{cases}
+\infty, & \Delta>0,\\
-\infty, & \Delta<0,\\
0, & \Delta=0.
\end{cases}
\]
Consequently,
\[
\psi_{\mathrm{clip},w}(X,Y,Z)
\longrightarrow
\psi_\infty(X,Y,Z),
\]
where
\[
\psi_\infty(X,Y,Z)
:=
\begin{cases}
1, & \Delta>0,\\
0, & \Delta<0,\\
\frac{1}{2}, & \Delta=0.
\end{cases}
\]
For the Gaussian kernel
\[
k(u,v)=\exp\bigl(-\|u-v\|^2\bigr),
\]
the function $t\mapsto e^{-t}$ is strictly decreasing. Hence
\[
\Delta>0
\iff
\|X-Y\|<\|Y-Z\|,
\]
\[
\Delta<0
\iff
\|X-Y\|>\|Y-Z\|,
\]
and
\[
\Delta=0
\iff
\|X-Y\|=\|Y-Z\|.
\]
Thus $\psi_\infty(X,Y,Z)$ is precisely the probability of predicting that
$Y$ was generated from $P$ under the 1-NN rule with uniform random
tie-breaking.
Let $L\in\{P,Q\}$ denote the distribution from which $Y$ was generated, so
that
\[
\mathbb{P}(L=P)=\mathbb{P}(L=Q)=\frac{1}{2},
\]
with
\[
Y\mid\{L=P\}\sim P,
\qquad
Y\mid\{L=Q\}\sim Q.
\]
For each $w>0$, define
\[
A_w
:=
\mathbf{1}_{\{L=P\}}\psi_{\mathrm{clip},w}(X,Y,Z)
+
\mathbf{1}_{\{L=Q\}}
\bigl(1-\psi_{\mathrm{clip},w}(X,Y,Z)\bigr).
\]
Conditional on $(X,Y,Z,L)$, the quantity $A_w$ is exactly the probability
that the clipped randomized classifier makes the correct prediction.
Therefore,
\[
\mathbb{E}[A_w]
\]
is the expected accuracy of that classifier.
Similarly, define
\[
A_\infty
:=
\mathbf{1}_{\{L=P\}}\psi_\infty(X,Y,Z)
+
\mathbf{1}_{\{L=Q\}}
\bigl(1-\psi_\infty(X,Y,Z)\bigr).
\]
Conditional on $(X,Y,Z,L)$, the quantity $A_\infty$ is exactly the
probability that the 1-NN rule with uniform random tie-breaking makes the
correct prediction. Hence
$\mathbb{E}[A_\infty]$
is the expected accuracy of that rule.
Since
\[
\psi_{\mathrm{clip},w}(X,Y,Z)
\longrightarrow
\psi_\infty(X,Y,Z)
\]
for every realization of $(X,Y,Z)$, it follows that
$A_w\longrightarrow A_\infty$
for every realization of $(X,Y,Z,L)$. Moreover,
$0\leq A_w\leq 1$
for every $w>0$. Thus the constant random variable $1$ is an integrable
dominating function.
For any sequence $(w_n)_{n\geq1}$ satisfying $w_n\to\infty$, the dominated
convergence theorem therefore gives
\[
\lim_{n\to\infty}\mathbb{E}[A_{w_n}]
=
\mathbb{E}[A_\infty].
\]
Since this holds for every sequence $w_n\to\infty$, we conclude that
\[
\lim_{w\to\infty}\mathbb{E}[A_w]
=
\mathbb{E}[A_\infty].
\]
Therefore, the expected accuracy of the clipped randomized rule converges
to the expected accuracy of the 1-NN rule with uniform random tie-breaking.
\hfill$\square$

 \subsection{Proof that the random question expected accuracy beats chance (from Section~\ref{beats2})}\label{proofbeyond}
 
\noindent\textbf{Proof.}
For each \(i \in \mathbb{N}\), let \(E_i\) denote the event that the
question \(A_i\) is selected so that
$\mathbb{P}(E_i) = \pi(i)$.
Fix \(i \in \mathbb{N}\), and write
$p_i := P(A_i)$ and
$q_i := Q(A_i)$.
Conditional on \(E_i\), the probability of correct classification is
the average of the success probabilities in the two cases where \(Y\)
is drawn from \(P\) and from \(Q\):
\[
\mathbb{P}(\mathrm{Correct} \mid E_i)
=
\frac{1}{2}
\mathbb{P}(\mathrm{Correct} \mid L=P, E_i)
+
\frac{1}{2}
\mathbb{P}(\mathrm{Correct} \mid L=Q, E_i).
\]
Suppose first that \(Y \sim P\). The classifier is correct in the
following three disjoint cases:
\begin{enumerate}
    \item \(Y \in A_i\), \(X \in A_i\), and \(Z \notin A_i\), which has
    probability
    $p_i^2(1-q_i)$.

    \item \(Y \notin A_i\), \(X \notin A_i\), and \(Z \in A_i\), which has
    probability
    $(1-p_i)^2 q_i$.

    \item The classifier uses its tie-breaking coin flip and predicts
    correctly. The coin flip is used exactly when \(X\) and \(Z\) give
    the same answer to the selected question. The probability of this
    event is
    $p_i q_i + (1-p_i)(1-q_i)$,
    and the independent fair coin predicts \(P\) correctly with
    probability \(1/2\). Hence, this case contributes
    $(1/2) [   p_i q_i + (1-p_i)(1-q_i)]$.
        
\end{enumerate}
Therefore,
\begin{align*}
\mathbb{P}(\mathrm{Correct} \mid L=P, E_i)
&=
p_i^2(1-q_i)
+
(1-p_i)^2 q_i
+
\frac{1}{2}
\bigl[
    p_i q_i + (1-p_i)(1-q_i)
\bigr]
\\
&=
(p_i^2-p_i^2 q_i)
+
(q_i-2p_i q_i+p_i^2 q_i)
+
\frac{1}{2}
\bigl(
    2p_i q_i+1-p_i-q_i
\bigr)
\\
&=
p_i^2-p_i q_i-\frac{1}{2}p_i+\frac{1}{2}q_i+\frac{1}{2}.
\end{align*}
By symmetry, interchanging \(p_i\) and \(q_i\) gives the corresponding
probability when \(Y \sim Q\):
\[
\mathbb{P}(\mathrm{Correct} \mid L=Q, E_i)
=
q_i^2-q_i p_i-\frac{1}{2}q_i+\frac{1}{2}p_i+\frac{1}{2}.
\]
Averaging the two conditional probabilities yields
\begin{align*}
\mathbb{P}(\mathrm{Correct} \mid E_i)
&=
\frac{1}{2}
\left[
    p_i^2+q_i^2-2p_i q_i+1
\right]
\\
&=
\frac{1}{2}
+
\frac{1}{2}(p_i-q_i)^2.
\end{align*}
Finally, applying the law of total probability over the selected
question gives
\begin{align*}
\mathbb{P}(\mathrm{Correct})
&=
\sum_{i=1}^{\infty}
\mathbb{P}(E_i)
\mathbb{P}(\mathrm{Correct} \mid E_i)
\\
&=
\sum_{i=1}^{\infty}
\pi(i)
\left[
    \frac{1}{2}
    +
    \frac{1}{2}(p_i-q_i)^2
\right]
\\
&=
\frac{1}{2}
+
\frac{1}{2}
\sum_{i=1}^{\infty}
\pi(i)
\bigl(
    P(A_i)-Q(A_i)
\bigr)^2.
\end{align*}
Because \(P \neq Q\) and \(\mathcal{G}\) is a generating
\(\pi\)-system for \(\mathcal{B}\), the uniqueness theorem for
probability measures implies that there exists \(k \in \mathbb{N}\)
such that
\[
P(A_k) \neq Q(A_k).
\]
Since \(\pi\) has full support, \(\pi(k)>0\), and hence
\[
\pi(k)
\bigl(
    P(A_k)-Q(A_k)
\bigr)^2
>0.
\]
Every term in the sum is nonnegative, and at least one term is positive, so the sum is strictly positive.
Therefore,
\[
\mathbb{P}(\mathrm{Correct})>\frac{1}{2}.
\]
$\hfill\square$

\vspace{1cm}

\subsection{Proof of the associated corollary on the binary agreement kernel (Section~\ref{beats2})}\label{binary}

\noindent\textit{1. Equivalence with the randomized kernel classifier.}
Fix $n\in\mathbb N$. Conditional on the event that the question $A_n$ is
selected and on the observations $(X,Y,Z)$, the probability that the
random-question classifier predicts $P$ is
\[
\frac{1}{2}
+
\frac{1}{2}
\bigl(
k_n(X,Y)-k_n(Y,Z)
\bigr).
\]
Indeed, this expression equals $1$ when $Y$ gives the same answer as $X$
but not as $Z$, equals $0$ when $Y$ gives the same answer as $Z$ but not
as $X$, and equals $1/2$ in all remaining cases.
Averaging over the independently selected index $n\sim\pi$ gives
\begin{align*}
\mathbb{P}(\text{predict }P\mid X,Y,Z)
&=
\sum_{n=1}^{\infty}
\pi(n)
\left[
\frac{1}{2}
+
\frac{1}{2}
\bigl(
k_n(X,Y)-k_n(Y,Z)
\bigr)
\right]
\\
&=
\frac{1}{2}
+
\frac{1}{2}
\bigl(
k_\pi(X,Y)-k_\pi(Y,Z)
\bigr).
\end{align*}
Hence the random-question classifier is exactly the randomized kernel
classifier associated with $k_\pi$.
Each $k_n$ is $(\mathcal B\otimes\mathcal B)$-measurable, being a sum of products of measurable coordinate indicators. Pointwise convergence of the partial sums defining $k_\pi$ proves its measurability.\\

\noindent\textit{2. Bounded, positive-semidefinite kernel.}
For each $n$, define $\phi_n:\Omega\to\{-1,+1\}$ by
$\phi_n(u)=2\mathbf{1}_{A_n}(u)-1$. A direct check gives
\[
k_n(u,v)
=
\mathbf{1}_{A_n}(u)\mathbf{1}_{A_n}(v)
+
\bigl(1-\mathbf{1}_{A_n}(u)\bigr)
\bigl(1-\mathbf{1}_{A_n}(v)\bigr)
=
\frac{1+\phi_n(u)\phi_n(v)}{2},
\]
so that
\[
k_\pi(u,v)
=
\sum_{n=1}^\infty \pi(n)k_n(u,v)
=
\frac12
+
\frac12\sum_{n=1}^\infty \pi(n)\phi_n(u)\phi_n(v)
=
\frac12
+
\frac12\bigl\langle \Phi(u),\Phi(v)\bigr\rangle_{\ell^2},
\]
where
\[
\Phi(u)
:=
\bigl(\sqrt{\pi(n)}\,\phi_n(u)\bigr)_{n\ge1}.
\]
Since $\phi_n(u)^2=1$ for every $u$ and $\sum_n\pi(n)=1$, we have
$\|\Phi(u)\|_{\ell^2}=1$ for every $u\in\Omega$. Two things follow
immediately:
\begin{itemize}
\item \textit{Boundedness}: by Cauchy--Schwarz,
\[
\bigl|\langle\Phi(u),\Phi(v)\rangle\bigr|
\leq
\|\Phi(u)\|\|\Phi(v)\|
=
1,
\]
so $k_\pi(u,v)\in[0,1]$ for all $u,v$, with $k_\pi(u,u)=1$.

\item \textit{Positive semidefiniteness}: for any
$u_1,\dots,u_m\in\Omega$ and $c_1,\dots,c_m\in\mathbb R$,
\[
\sum_{i,j}c_ic_j\,k_\pi(u_i,u_j)
=
\frac12\left(\sum_i c_i\right)^2
+
\frac12
\left\|
\sum_i c_i\Phi(u_i)
\right\|_{\ell^2}^2
\geq
0.
\]
\end{itemize}
This also guarantees that every expectation below is finite for arbitrary
$P,Q$, and, since each summand $\pi(n)k_n(u,v)\geq0$, allows us to
interchange $\sum_n$ and $\mathbb E[\cdot]$ by Tonelli's theorem.

\medskip
\noindent\textit{3. Term-by-term MMD.}
For a fixed set $A_n$, write
\[
p_n=P(A_n),
\qquad
q_n=Q(A_n).
\]
The expected values are
\begin{align*}
\mathbb{E}_{P\times P}[k_n]
&=
p_n^2+(1-p_n)^2
=
2p_n^2-2p_n+1,
\\
\mathbb{E}_{P\times Q}[k_n]
&=
p_nq_n+(1-p_n)(1-q_n)
=
2p_nq_n-p_n-q_n+1,
\\
\mathbb{E}_{Q\times Q}[k_n]
&=
q_n^2+(1-q_n)^2
=
2q_n^2-2q_n+1.
\end{align*}
Substituting these into the MMD expansion yields
\[
\mathbb{E}_{P\times P}[k_n]
-
2\mathbb{E}_{P\times Q}[k_n]
+
\mathbb{E}_{Q\times Q}[k_n]
=
2p_n^2-4p_nq_n+2q_n^2
=
2(p_n-q_n)^2.
\]
By the Tonelli justification in Step 2, summing over $n$ weighted by
$\pi(n)$ gives
\[
\mathrm{MMD}_{k_\pi}^2(P,Q)
=
2\sum_{n=1}^\infty
\pi(n)(p_n-q_n)^2.
\]
Substituting this identity into Theorem~\ref{theo:borel_accuracy} gives
\[
\mathbb{P}(\mathrm{Correct})
=
\frac{1}{2}
+
\frac{1}{4}\mathrm{MMD}_{k_\pi}^2(P,Q).
\]

\smallskip
\noindent\textit{4. Characteristic kernel.}
If $P=Q$, then clearly
\[
\mathrm{MMD}_{k_\pi}(P,Q)=0.
\]
Conversely, suppose $P\neq Q$. Because $\mathcal G$ is a generating
$\pi$-system for $\mathcal B$, there exists $n\in\mathbb N$ such that
\[
P(A_n)\neq Q(A_n).
\]
Since $\pi$ has full support,
\[
\pi(n)\bigl(P(A_n)-Q(A_n)\bigr)^2>0.
\]
It follows from
\[
\mathrm{MMD}_{k_\pi}^2(P,Q)
=
2\sum_{n=1}^{\infty}
\pi(n)\bigl(P(A_n)-Q(A_n)\bigr)^2
\]
that
\[
\mathrm{MMD}_{k_\pi}(P,Q)>0.
\]
Hence
\[
\mathrm{MMD}_{k_\pi}(P,Q)=0
\iff
P=Q,
\]
so $k_\pi$ is characteristic.
\hfill$\square$


\subsection{Random questions for predicting graph distribution class (from Section~\ref{beats2})}\label{Graphs}

This example considers the finite space of all labeled graphs on a fixed set of $n$ vertices. Let $V = \{1, 2, \dots, n\}$ be the vertex set, and let $E_{\max}$ be the set of all $m = \binom{n}{2}$ possible edges. The sample space $\Omega$ is the set of all $2^m$ possible labeled graphs on $V$, equipped with the discrete metric and its corresponding Borel $\sigma$-algebra, namely the power set $2^\Omega$.
For any subset of edges $S \subseteq E_{\max}$, define the event $A_S$ as the set of all graphs that contain every edge in $S$:
\[
A_S
=
\{G \in \Omega : S \subseteq E(G)\}.
\]
We define our dictionary of questions as the collection of all such edge-inclusion events:
\[
\mathcal{G}_{\mathrm{inc}}
=
\{A_S : S \subseteq E_{\max}\}.
\]
This collection $\mathcal{G}_{\mathrm{inc}}$ satisfies the two necessary properties required to apply the random-question framework.
First, $\mathcal{G}_{\mathrm{inc}}$ is a $\pi$-system. Because the vertices are labeled, the edges have fixed identities. If a graph $G$ contains the edge set $S_1$ and the edge set $S_2$, then it contains their union. Thus, the intersection of any two sets in $\mathcal{G}_{\mathrm{inc}}$ is another set in the collection:
\[
A_{S_1} \cap A_{S_2}
=
\{G \in \Omega : S_1 \subseteq E(G) \text{ and } S_2 \subseteq E(G)\}
=
A_{S_1 \cup S_2}.
\]
Second, $\mathcal{G}_{\mathrm{inc}}$ generates the Borel $\sigma$-algebra. It is sufficient to show that the $\sigma$-algebra generated by $\mathcal{G}_{\mathrm{inc}}$ contains every singleton graph $\{G\}$. Let $G$ be a specific labeled graph with edge set $E(G)$. The event $A_{E(G)}$ contains $G$ together with every supergraph of $G$. We isolate the singleton $\{G\}$ by removing the events corresponding to adding at least one edge from $E_{\max} \setminus E(G)$:
\[
\{G\}
=
A_{E(G)}
\setminus
\left(
\bigcup_{e \in E_{\max} \setminus E(G)}
A_{E(G) \cup \{e\}}
\right).
\]
Since $\sigma$-algebras are closed under finite unions and complements, it follows that $\{G\} \in \sigma(\mathcal{G}_{\mathrm{inc}})$. Hence $\mathcal{G}_{\mathrm{inc}}$ generates the full power set:
$\sigma(\mathcal{G}_{\mathrm{inc}})
=
2^\Omega$.
Consequently, if $P$ and $Q$ are two distinct probability distributions over labeled graphs on $V$, then there exists at least one edge subset $S^* \subseteq E_{\max}$ such that
$P(A_{S^*}) \neq Q(A_{S^*})$.
Fix $r \in (0,1/2)$ and define a full-support probability mass function $\pi$ on the $2^m$ sets in $\mathcal{G}_{\mathrm{inc}}$ by
\[
\pi(A_S)
=
r^{|S|}(1-r)^{m-|S|},
\qquad
S \subseteq E_{\max}.
\]
Equivalently, a random question is generated by including each possible edge independently with probability $r$. Every set $A_S \in \mathcal{G}_{\mathrm{inc}}$ therefore receives positive probability, while choosing $r<1/2$ places greater weight on questions involving smaller labeled edge sets. Theorem~\ref{theo:borel_accuracy} consequently guarantees that the random-question classifier has expected accuracy strictly greater than $1/2$ whenever $P \neq Q$.\\

\noindent\textbf{Illustrative example and exact benchmark.}
Let $P=\operatorname{ER}(n,p)$ and $Q=\operatorname{ER}(n,q)$ be Erd\H{o}s--R\'enyi graphs \cite{erdos1959random,gilbert1959random}, with $m=\binom n2$, and include each possible query edge independently with probability $r$. If $M=|S|$, then $M\sim\operatorname{Binomial}(m,r)$ and
\[
 P(A_S)=p^M,\qquad Q(A_S)=q^M.
\]
By Theorem~\ref{theo:borel_accuracy} and the binomial identity $\mathbb E[t^M]=(1-r+rt)^m$,
\begin{equation}\label{eq:graph-exact}
\begin{split}
 \operatorname{Acc}
 =\frac12+\frac12\big[&(1-r+rp^2)^m+(1-r+rq^2)^m\\
                     &-2(1-r+rpq)^m\big].
\end{split}
\end{equation}
The probability $I$ that the sampled question distinguishes the references is
\[
 I=(1-r+rp)^m+(1-r+rq)^m-2(1-r+rpq)^m,
\]
because conditional on $S$ that probability is $p^M(1-q^M)+(1-p^M)q^M$. When $I>0$, the accuracy on these trials is
\[
 \mathbb P(\mathrm{Correct}\mid I_X\neq I_Z)
   =\frac12+\frac{\operatorname{Acc}-1/2}{I},
\]
since all other trials are resolved by a fair coin. For $n=25$, $m=300$, $p=0.4$, $q=0.6$, and $r=0.005$, the expected query size is $mr=1.5$ and these exact formulas give
\mbox{$\operatorname{Acc} \approx 0.5134778754$,} $I \approx 0.3162514960$, and
 $\mathbb P(\mathrm{Correct}\mid I_X\neq I_Z) \approx0.5426175864$.
These are population benchmarks and do not rely on simulation.

This Erd\H{o}s--R\'enyi example is deliberately simple. In this
setting, the two graph distributions differ only in their edge
probabilities, and the number of edges in a graph is therefore highly
informative about its class. Indeed, when the two edge probabilities are
known, the Bayes classifier depends on \(Y\) only through \(|E(Y)|\) and
 has accuracy close to one for $n=25$, $p=0.4$, and $q=0.6$. The likelihood ratio is monotone in the edge count and crosses $1$ at $150$ edges. By symmetry, with $B\sim\operatorname{Binomial}(300,0.4)$ the Bayes accuracy is $\mathbb P(B<150)+\tfrac12\mathbb P(B=150)\approx0.999764394$. Thus, an
edge-count-based classifier can substantially outperform the
random-question rule in this particular example.

The purpose of the random-question rule is different. It does not assume
that the graph distributions are Erd\H{o}s--R\'enyi, or that edge count is
the relevant statistic. For example, two distinct graph distributions
may have the same edge-count distribution while assigning different
probabilities to particular labeled edge configurations. In that case,
the joint distribution of
$\bigl(|E(X)|,|E(Y)|,|E(Z)|\bigr)$
does not depend on the class generating \(Y\), and hence every classifier
based only on these edge counts has expected accuracy \(1/2\) under equal
class priors. By contrast, because the edge-inclusion events in
\(\mathcal G_{\mathrm{inc}}\) generate the full power set \(2^\Omega\), a
random-question classifier whose sampling distribution has full support
on \(\mathcal G_{\mathrm{inc}}\) has expected accuracy strictly greater
than \(1/2\) for every pair of distinct distributions \(P\neq Q\).

An explicit equal-edge-count example is as follows: take three labeled vertices and let $P$ put all mass on the graph with the single edge $\{1,2\}$, while $Q$ puts all mass on the graph with the single edge $\{1,3\}$. Every observation has one edge, so edge counts contain no class information. Only the query sets $S=\{\{1,2\}\}$ and $S=\{\{1,3\}\}$ distinguish the two laws; each has probability $r(1-r)^2$. The random-question accuracy is therefore $1/2+r(1-r)^2>1/2$.


\subsection{Proof that the deterministic order rule beats chance (from Section~\ref{subsec:order_rule})}\label{ordered}

\noindent\textbf{Proof of the absolutely continuous lemma.}
Define
$D=F-G$ 
and
$S=F+G$.
Since \(P\) and \(Q\) are absolutely continuous, \(D\) and \(S\) are
absolutely continuous, with
$D'=f-g$ and
$S'=f+g$
almost everywhere.
Let
\[
\pi_P
=
\mathbb P\bigl(C^{\mathrm{ord}}(X,Y,Z)=P\mid L=P\bigr)
\]
and
\[
\pi_Q
=
\mathbb P\bigl(C^{\mathrm{ord}}(X,Y,Z)=P\mid L=Q\bigr).
\]
When \(Y\sim P\), conditioning on \(X\) in the event \(X<Z\), and on
\(Z\) in the event \(X>Z\), gives
\[
\pi_P
=
\int_{\mathbb R}
F(t)\bigl(1-G(t)\bigr)f(t)\,dt
+
\int_{\mathbb R}
\bigl(1-F(t)\bigr)^2g(t)\,dt.
\]
Similarly, when \(Y\sim Q\),
\[
\pi_Q
=
\int_{\mathbb R}
G(t)\bigl(1-G(t)\bigr)f(t)\,dt
+
\int_{\mathbb R}
\bigl(1-F(t)\bigr)\bigl(1-G(t)\bigr)g(t)\,dt.
\]
The overall probability of correct classification is
\[
\mathbb P(\mathrm{Correct})
=
\frac12\pi_P+\frac12(1-\pi_Q).
\]
Consequently,
\begin{align*}
2\mathbb P(\mathrm{Correct})-1
&=
\pi_P-\pi_Q
\\
&=
\int_{\mathbb R}
D(t)\bigl(1-G(t)\bigr)f(t)\,dt
-
\int_{\mathbb R}
D(t)\bigl(1-F(t)\bigr)g(t)\,dt
\\
&=
\int_{\mathbb R}D(t)\bigl(f(t)-g(t)\bigr)\,dt
+
\int_{\mathbb R}
D(t)\bigl(F(t)g(t)-G(t)f(t)\bigr)\,dt.
\end{align*}
Since \(D'=f-g\), the first term is
\[
\int_{\mathbb R}D(t)D'(t)\,dt
=
\frac12
\bigl[D(t)^2\bigr]_{-\infty}^{+\infty}
=
0,
\]
because
$D(-\infty)=D(+\infty)=0$.
Moreover,
\[
F=\frac{S+D}{2},
\qquad
G=\frac{S-D}{2},
\qquad
f=\frac{S'+D'}{2},
\qquad
g=\frac{S'-D'}{2},
\]
and hence
\[
Fg-Gf
=
\frac12\bigl(DS'-SD'\bigr)
\]
almost everywhere. It follows that
\begin{align*}
2\mathbb P(\mathrm{Correct})-1
&=
\frac12
\int_{\mathbb R}D(t)^2S'(t)\,dt
-
\frac12
\int_{\mathbb R}S(t)D(t)D'(t)\,dt
\\
&=
\frac12
\int_{\mathbb R}D(t)^2S'(t)\,dt
-
\frac14
\int_{\mathbb R}S(t)\bigl(D(t)^2\bigr)'\,dt.
\end{align*}
Ordinary integration by parts gives
\[
\int_{\mathbb R}S(t)\bigl(D(t)^2\bigr)'\,dt
=
\bigl[S(t)D(t)^2\bigr]_{-\infty}^{+\infty}
-
\int_{\mathbb R}S'(t)D(t)^2\,dt.
\]
The boundary term vanishes because \(D(t)\to0\) as
\(t\to\pm\infty\), while \(0\leq S(t)\leq2\). Therefore,
\[
\int_{\mathbb R}S(t)\bigl(D(t)^2\bigr)'\,dt
=
-
\int_{\mathbb R}S'(t)D(t)^2\,dt.
\]
Substituting this identity gives
\[
2\mathbb P(\mathrm{Correct})-1
=
\frac34
\int_{\mathbb R}D(t)^2S'(t)\,dt.
\]
Since \(S'=f+g\) almost everywhere,
\[
\mathbb P(\mathrm{Correct})
=
\frac12
+
\frac38
\int_{\mathbb R}
D(t)^2\bigl(f(t)+g(t)\bigr)\,dt.
\]
It remains to prove strict positivity when \(P\neq Q\). In that case
\(F\neq G\), so there exists \(x\in\mathbb R\) for which
\(D(x)\neq0\). The cases \(D(x)>0\) and \(D(x)<0\) are symmetric, so
suppose that \(D(x)>0\), and set
\[
c=\frac{D(x)}{2}>0.
\]
Since \(D(t)\to0\) as \(t\to-\infty\), continuity of \(D\) implies that
the set
\[
\{t\leq x:D(t)=c\}
\]
is nonempty. Let
$a=\sup\{t\leq x:D(t)=c\}$.
Then \(a<x\), \(D(a)=c\), and
$D(t)>c$
for every $t\in(a,x]$.
Furthermore,
\begin{align*}
(P-Q)((a,x])
&=
F(x)-F(a)-G(x)+G(a)
\\
&=
D(x)-D(a)
\\
&=
c
>
0.
\end{align*}
It follows that
$(P+Q)((a,x])>0$.
Consequently,
\begin{align*}
\int_{\mathbb R}
D(t)^2\,d(P+Q)(t)
&\geq
\int_{(a,x]}
D(t)^2\,d(P+Q)(t)
\\
&\geq
c^2(P+Q)((a,x])
\\
&>
0.
\end{align*}
Therefore,
\[
\mathbb P(\mathrm{Correct})>\frac12.
\]
\hfill$\square$

\vspace{1cm}

\noindent\textbf{The deterministic order rule for arbitrary distributions.}
We now derive the exact accuracy of the deterministic order rule when
\(P\) and \(Q\) are arbitrary probability distributions on
\(\mathbb R\). In this setting, atoms may be present, so the precise
placement of the strict and weak inequalities in the rule matters.
Let \(F\) and \(G\) denote the cumulative distribution functions of
\(P\) and \(Q\). For any right-continuous function \(H\) with left
limits, write
\[
H_-(t)=\lim_{s\uparrow t}H(s)
\]
and
\[
\Delta H(t)=H(t)-H_-(t).
\]
Define
\[
D(t)=F(t)-G(t)
\qquad\text{and}\qquad
D_-(t)=F_-(t)-G_-(t).
\]
Thus,
\[
\Delta D(t)
=
P(\{t\})-Q(\{t\}).
\]
Recall that the order rule predicts \(P\) precisely on the event
\[
\{X<Z,\ Y\leq X\}
\cup
\{X>Z,\ Y>Z\},
\]
and predicts \(Q\) otherwise.

\begin{theo}
\label{theo:order_rule}
Let \(P\) and \(Q\) be arbitrary probability distributions on
\(\mathbb R\), and let
\[
X\sim P,
\qquad
Z\sim Q,
\qquad
Y\sim\frac12P+\frac12Q
\]
be independent. Then the probability of correct classification of the
deterministic order rule is
\begin{align}
\mathbb P(\mathrm{Correct})
=
\frac12
&+
\frac14
\int_{\mathbb R}
D(t)^2\,d(P+Q)(t)
\notag\\
&+
\frac18
\int_{\mathbb R}
D_-(t)^2\,d(P+Q)(t)
\notag\\
&+
\frac18
\sum_{t\in\mathbb R}
\bigl(2-F(t)-G(t)\bigr)
\bigl(\Delta D(t)\bigr)^2.
\label{eq:order_rule_general_exact}
\end{align}
Only countably many terms in the sum are nonzero, and the sum is
finite. Moreover,
\[
\mathbb P(\mathrm{Correct})=\frac12
\]
if and only if \(P=Q\).
\end{theo}

\vspace{0.5cm}

\noindent\textbf{Proof.}
All integrals below are Lebesgue--Stieltjes integrals. Let
\[
\pi_P
=
\mathbb P\bigl(
C^{\mathrm{ord}}(X,Y,Z)=P
\mid L=P
\bigr)
\]
and
\[
\pi_Q
=
\mathbb P\bigl(
C^{\mathrm{ord}}(X,Y,Z)=P
\mid L=Q
\bigr).
\]
Suppose first that \(Y\sim P\). On the event \(X<Z\), the rule predicts
\(P\) when \(Y\leq X\). Conditioning on \(X=t\), this event has
probability
\[
F(t)\bigl(1-G(t)\bigr).
\]
On the event \(X>Z\), the rule predicts \(P\) when \(Y>Z\).
Conditioning on \(Z=t\), this event has probability
\[
\bigl(1-F(t)\bigr)^2.
\]
Therefore,
\[
\pi_P
=
\int_{\mathbb R}
F(t)\bigl(1-G(t)\bigr)\,dP(t)
+
\int_{\mathbb R}
\bigl(1-F(t)\bigr)^2\,dQ(t).
\]
Similarly, when \(Y\sim Q\),
\[
\pi_Q
=
\int_{\mathbb R}
G(t)\bigl(1-G(t)\bigr)\,dP(t)
+
\int_{\mathbb R}
\bigl(1-F(t)\bigr)\bigl(1-G(t)\bigr)\,dQ(t).
\]
Since the two possible target classes have equal prior probability,
\[
\mathbb P(\mathrm{Correct})
=
\frac12\pi_P+\frac12(1-\pi_Q).
\]
Consequently, if
\[
J=2\mathbb P(\mathrm{Correct})-1,
\]
then
\begin{align}
J
&=
\pi_P-\pi_Q
\notag\\
&=
\int_{\mathbb R}
D(t)\bigl(1-G(t)\bigr)\,dF(t)
-
\int_{\mathbb R}
D(t)\bigl(1-F(t)\bigr)\,dG(t).
\label{eq:order_general_advantage}
\end{align}
Define
\[
S=F+G.
\]
Expanding \eqref{eq:order_general_advantage} gives
\begin{align*}
J
&=
\int_{\mathbb R}D\,dF
-
\int_{\mathbb R}DG\,dF
-
\int_{\mathbb R}D\,dG
+
\int_{\mathbb R}DF\,dG
\\
&=
\int_{\mathbb R}D\,dD
+
\int_{\mathbb R}D(F\,dG-G\,dF).
\end{align*}
Since
\[
F=\frac12(S+D),
\qquad
G=\frac12(S-D),
\]
linearity of the corresponding Lebesgue--Stieltjes measures gives
\[
F\,dG-G\,dF
=
\frac12(D\,dS-S\,dD).
\]
Hence
\begin{equation}
\label{eq:order_general_clean}
J
=
\int_{\mathbb R}D\,dD
+
\frac12\int_{\mathbb R}D^2\,dS
-
\frac12\int_{\mathbb R}SD\,dD.
\end{equation}
We now evaluate the first and third terms in
\eqref{eq:order_general_clean}. For right-continuous functions of
bounded variation, integration by parts gives
\[
d(UV)=U\,dV+V_-\,dU.
\]
Applying this identity to \(D^2\) yields
\[
d(D^2)
=
(D+D_-)\,dD
=
(2D-\Delta D)\,dD.
\]
Because
\[
D(-\infty)=D(+\infty)=0,
\]
integration over \(\mathbb R\) gives
\[
0
=
2\int_{\mathbb R}D\,dD
-
\int_{\mathbb R}\Delta D\,dD.
\]
The function \(\Delta D\) is nonzero only at the jump points of \(D\),
and
\[
dD(\{t\})=\Delta D(t).
\]
It follows that
\begin{equation}
\label{eq:order_general_DdD}
\int_{\mathbb R}D\,dD
=
\frac12
\sum_{t\in\mathbb R}
\bigl(\Delta D(t)\bigr)^2.
\end{equation}
Next, applying integration by parts to the product \(SD^2\) gives
\[
d(SD^2)
=
S\,d(D^2)+D_-^2\,dS.
\]
Using
\[
d(D^2)=(2D-\Delta D)\,dD
\]
and
\[
\bigl[SD^2\bigr]_{-\infty}^{+\infty}=0,
\]
we obtain
\[
0
=
2\int_{\mathbb R}SD\,dD
-
\int_{\mathbb R}S\Delta D\,dD
+
\int_{\mathbb R}D_-^2\,dS.
\]
Therefore,
\begin{align}
2\int_{\mathbb R}SD\,dD
&=
\int_{\mathbb R}S\Delta D\,dD
-
\int_{\mathbb R}D_-^2\,dS
\notag\\
&=
\sum_{t\in\mathbb R}
S(t)\bigl(\Delta D(t)\bigr)^2
-
\int_{\mathbb R}D_-(t)^2\,dS(t).
\label{eq:order_general_SDdD}
\end{align}
Substituting \eqref{eq:order_general_DdD} and
\eqref{eq:order_general_SDdD} into
\eqref{eq:order_general_clean} gives
\begin{align*}
J
&=
\frac12
\sum_t\bigl(\Delta D(t)\bigr)^2
+
\frac12
\int_{\mathbb R}D(t)^2\,dS(t)
\\
&\qquad
-
\frac14
\sum_tS(t)\bigl(\Delta D(t)\bigr)^2
+
\frac14
\int_{\mathbb R}D_-(t)^2\,dS(t)
\\
&=
\frac12
\int_{\mathbb R}D(t)^2\,dS(t)
+
\frac14
\int_{\mathbb R}D_-(t)^2\,dS(t)
\\
&\qquad
+
\frac14
\sum_t
\bigl(2-S(t)\bigr)
\bigl(\Delta D(t)\bigr)^2.
\end{align*}
Since
\[
dS=d(P+Q),
\qquad
S=F+G,
\]
and
\[
\mathbb P(\mathrm{Correct})
=
\frac12+\frac12J,
\]
we obtain \eqref{eq:order_rule_general_exact}.
We next verify that the jump sum is well defined and finite. The set of
jump points of \(D\) is contained in the union of the sets of atoms of
\(P\) and \(Q\), and is therefore at most countable. Moreover,
\[
0\leq 2-F(t)-G(t)\leq2.
\]
Writing
\[
\Delta F(t)=P(\{t\}),
\qquad
\Delta G(t)=Q(\{t\}),
\]
we have
\[
\bigl(\Delta D(t)\bigr)^2
=
\bigl(\Delta F(t)-\Delta G(t)\bigr)^2
\leq
\Delta F(t)+\Delta G(t).
\]
Consequently,
\begin{align*}
\sum_t
\bigl(2-F(t)-G(t)\bigr)
\bigl(\Delta D(t)\bigr)^2
&\leq
2\sum_t
\bigl(\Delta F(t)+\Delta G(t)\bigr)
\\
&\leq4.
\end{align*}
Thus the sum is finite.
All three terms added to \(1/2\) in
\eqref{eq:order_rule_general_exact} are nonnegative. If \(P=Q\), then
\(D=0\), and hence
\[
\mathbb P(\mathrm{Correct})=\frac12.
\]
Conversely, suppose that
\[
\mathbb P(\mathrm{Correct})=\frac12.
\]
The nonnegativity of the terms in
\eqref{eq:order_rule_general_exact} implies
\[
\int_{\mathbb R}D(t)^2\,d(P+Q)(t)=0
\]
and
\[
\int_{\mathbb R}D_-(t)^2\,d(P+Q)(t)=0.
\]
Let
\[
\mu=P+Q.
\]
It follows that
\[
D=0
\qquad\text{and}\qquad
D_-=0
\]
\(\mu\)-almost everywhere.
If \(\Delta D(t)\neq0\), then \(t\) is an atom of \(P\) or \(Q\), so
\[
\mu(\{t\})>0.
\]
The preceding almost-everywhere equalities would therefore imply
\[
D(t)=D_-(t)=0,
\]
contradicting
\[
\Delta D(t)=D(t)-D_-(t)\neq0.
\]
Thus \(D\) has no jumps and is continuous.
Let
\[
U=\{t\in\mathbb R:D(t)\neq0\}.
\]
Since \(D\) is continuous, \(U\) is open. Since \(D=0\)
\(\mu\)-almost everywhere,
\[
\mu(U)=0.
\]
Let \(I\) be a connected component of \(U\). For any \(s<t\) in \(I\),
we have
\[
(s,t]\subseteq I
\]
and hence
\[
\mu((s,t])=0.
\]
Therefore,
\[
D(t)-D(s)
=
(P-Q)((s,t])
=
0.
\]
Thus \(D\) is constant on \(I\). Since \(I\subseteq U\), this constant
would have to be nonzero. If \(I\) has a finite endpoint, continuity of
\(D\) forces the constant to be zero at that endpoint. If \(I\) is
unbounded, the same conclusion follows from
\[
D(-\infty)=D(+\infty)=0.
\]
Both possibilities are contradictions. Hence \(U\) is empty, so
\[
D=0
\]
everywhere. Therefore \(F=G\), and consequently \(P=Q\).
This proves that equality holds if and only if \(P=Q\), and completes
the proof.
\hfill $\square$


\subsection{Proof that no fixed rule beats $1/2$ for an unknown unbalanced prior on the class of $Y$ (from 
Section~\ref{subsec:unknown-adversarial})}\label{Alpha}
 
Suppose, toward a contradiction, that such a rule $\psi$ exists.
Fix two distinct points $u,v\in\Omega$ and, for $p,q\in(0,1)$, let
\[
P_p:=p\delta_u+(1-p)\delta_v,
\qquad
Q_q:=q\delta_u+(1-q)\delta_v.
\]
Then $P_p\neq Q_q$ whenever $p\neq q$. Write
\[
b_P(p,q):=B_P(\psi;P_p,Q_q),
\qquad
b_Q(p,q):=B_Q(\psi;P_p,Q_q).
\]
For every $p\neq q$, the assumed guarantee gives
\[
\alpha b_P(p,q)+(1-\alpha)b_Q(p,q)>\frac12
\]
for every $\alpha\in(0,1)\setminus\{1/2\}$. Letting
$\alpha\uparrow1$ and $\alpha\downarrow0$ yields
\begin{equation}
b_P(p,q)\geq\frac12,
\qquad
b_Q(p,q)\geq\frac12
\qquad (p\neq q).
\label{eq:alpha-classwise}
\end{equation}
Both functions are continuous in $(p,q)$, so these inequalities also
hold when $p=q$. But when $p=q$, all three observations have the same
distribution, and hence
\[
b_P(p,p)+b_Q(p,p)=1.
\]
Therefore
\[
b_P(p,p)=b_Q(p,p)=\frac12.
\]
Now fix $p\in(0,1)$. In $b_P(p,q)$, the parameter $q$ enters only
through $Z\sim Q_q$, so $q\mapsto b_P(p,q)$ is affine. By
\eqref{eq:alpha-classwise},
\[
b_P(p,q)\geq\frac12
\]
for $q\neq p$, while $b_P(p,p)=1/2$. An affine function attaining its
minimum at an interior point must be constant, so
\[
b_P(p,q)=\frac12
\qquad\text{for every }q\in(0,1).
\]
Similarly, for fixed $q$, the function $p\mapsto b_Q(p,q)$ is affine,
and therefore
\[
b_Q(p,q)=\frac12
\qquad\text{for every }p\in(0,1).
\]
Thus, for every distinct $p,q$ and every $\alpha\in(0,1)$,
\[
\operatorname{Acc}_{\alpha}(\psi;P_p,Q_q)
=
\alpha\frac12+(1-\alpha)\frac12
=
\frac12,
\]
contradicting the assumed strict inequality.
\hfill$\square$


\subsection{Proof that an adversary cannot be uniformly beaten (from 
Section~\ref{subsec:unknown-adversarial})}\label{adversary}

\noindent\textbf{Proof.}
We prove the contrapositive of Theorem~\ref{theo:adv-no-universal}. Suppose that
\[
\operatorname{Acc}_{\mathrm{adv}}(\psi;P,Q)
\geq \frac12
\]
for every pair of distinct finitely supported probability measures
\(P\neq Q\) on \(\Omega\).
Fix an arbitrary finite set
\[
S=\{s_1,\ldots,s_m\}\subseteq\Omega,
\qquad m\geq2.
\]
For probability vectors
\[
p=(p_1,\ldots,p_m),
\qquad
q=(q_1,\ldots,q_m),
\]
define
\[
P_p:=\sum_{i=1}^m p_i\delta_{s_i},
\qquad
Q_q:=\sum_{i=1}^m q_i\delta_{s_i}.
\]
Since the points \(s_i\) are distinct and their singletons are
measurable, \(P_p=Q_q\) if and only if \(p=q\).
The adversarial accuracy corresponding to these measures is
\begin{align}
A_S(p,q)
:={}&
\operatorname{Acc}_{\mathrm{adv}}(\psi;P_p,Q_q)
\notag\\
={}&
\sum_{i,k=1}^m p_iq_k
\min\left\{
\sum_{j=1}^m p_j\psi(s_i,s_j,s_k),
\sum_{j=1}^m q_j\bigl(1-\psi(s_i,s_j,s_k)\bigr)
\right\}.
\label{eq:adv-finite-simplex}
\end{align}
The function \(A_S\) is continuous in \(p\) and \(q\), since it is a
finite sum formed from additions, multiplications, and the minimum of two
linear functions.
Now let
$r=(r_1,\ldots,r_m)$
be any probability vector with \(r_j>0\) for every \(j\), and define
$R:=P_r=Q_r$.
Choose a sequence \(\varepsilon_n>0\) such that
\[
\varepsilon_n\longrightarrow0
\qquad\text{and}\qquad
\varepsilon_n<\min\{r_1,r_2\}.
\]
Set
\[
q^{(n)}
=
r+\varepsilon_n(e_1-e_2),
\]
where \(e_1,e_2\) are the first two standard basis vectors. Each
\(q^{(n)}\) is a probability vector, \(q^{(n)}\neq r\), and
$q^{(n)}\longrightarrow r$.
By the assumed lower bound,
\[
A_S(r,q^{(n)})\geq\frac12
\]
for every \(n\). Continuity of \(A_S\) therefore gives
\[
A_S(r,r)\geq\frac12.
\]
For \(x,z\in S\), define
\[
g_R(x,z)
:=
\int\psi(x,y,z)\,dR(y).
\]
When \(P=Q=R\), choosing either label produces the same distribution for
\(Y\). Hence
\[
a_P(x,z)=g_R(x,z),
\qquad
a_Q(x,z)=1-g_R(x,z).
\]
It follows that
\begin{align}
\operatorname{Acc}_{\mathrm{adv}}(\psi;R,R)
&=
\sum_{i,k=1}^m
r_ir_k
\min\{g_R(s_i,s_k),1-g_R(s_i,s_k)\}
\notag\\
&=
\frac12
-
\sum_{i,k=1}^m
r_ir_k
\left|
g_R(s_i,s_k)-\frac12
\right|
\leq
\frac12.
\label{eq:adv-equal-distributions}
\end{align}
Together with the preceding lower bound, this forces equality. Since
\(r_ir_k>0\) for every \(i,k\), we obtain
\[
g_R(s_i,s_k)=\frac12
\qquad
\text{for every }i,k.
\]
Equivalently,
\begin{equation}
\sum_{j=1}^m
r_j\psi(s_i,s_j,s_k)
=
\frac12
\qquad
\text{for every }i,k.
\label{eq:adv-affine-identity}
\end{equation}
This identity holds for every strictly positive probability vector \(r\).
Fix \(i,k\), and choose two distinct indices \(j,\ell\). Starting from
any strictly positive probability vector \(r\), replace it by
$r+\varepsilon(e_j-e_\ell)$
for sufficiently small nonzero \(\varepsilon\). Applying
Equation~\eqref{eq:adv-affine-identity} to both probability vectors and
subtracting gives
\[
\psi(s_i,s_j,s_k)
=
\psi(s_i,s_\ell,s_k).
\]
Thus, for fixed \(i,k\), all the values
\[
\psi(s_i,s_1,s_k),\ldots,\psi(s_i,s_m,s_k)
\]
are equal. Equation~\eqref{eq:adv-affine-identity} then implies that their
common value is \(1/2\). Therefore,
\[
\psi(s_i,s_j,s_k)=\frac12
\qquad
\text{for every }i,j,k.
\]
Finally, any prescribed triple \((x,y,z)\in\Omega^3\) is contained in
some finite set \(S\subseteq\Omega\) having at least two points. Applying
the preceding argument to that set gives
\[
\psi(x,y,z)=\frac12.
\]
Since the triple was arbitrary,
$\psi\equiv 1/2$ on
$\Omega^3$.
This proves the contrapositive. Conversely, $\psi\equiv1/2$ gives $a_P=a_Q=1/2$ for every pair of laws, establishing the stated equivalence.
\hfill\(\square\)

\end{document}